\documentclass[letterpaper,10pt,journal,twoside]{IEEEtran}
\usepackage{amsmath,amsfonts}
\usepackage{algorithm}
\usepackage{array}
\usepackage[caption=false,font=normalsize,labelfont=sf,textfont=sf]{subfig}
\usepackage{textcomp}
\usepackage{stfloats}
\usepackage{url}
\usepackage{hyperref}
\usepackage{verbatim}
\usepackage{graphicx}
\usepackage{balance}
\usepackage{algpseudocode}
\usepackage{subcaption}
\usepackage{tikz}
\usepackage{booktabs}
\usepackage[svgnames]{xcolor}
\usepackage{makecell}
\usepackage{wrapfig}
\usepackage{soul}
\soulregister{\eqref}{7}
\usepackage[font={footnotesize}]{caption}
\usepackage{cite}

\newcommand{\figref}[1]{{Fig.~\ref{#1}}}
\newcommand{\tabref}[1]{Table~\ref{#1}}
\newcommand{\secref}[1]{Section~\ref{#1}}
\newcommand{\pseref}[1]{Algorithm~\ref{#1}}
\newcommand{\theoref}[1]{Theorem~\ref{#1}}
\newcommand{\bd}[1]{\boldsymbol{#1}}

\usepackage{amsthm}
\makeatletter
\newtheoremstyle{ieeebreak}%
  {6pt}
  {6pt}
  {\itshape}
  {}
  {\bfseries}
  {}
  {0pt}
  {%
    \thmname{#1}\thmnumber{ #2}\thmnote{ {\normalfont\textit{(#3)}}}\newline
  }%
\makeatother
\theoremstyle{ieeebreak}
\newtheorem{theorem}{Theorem}[section]
\newtheorem{lemma}[theorem]{Lemma}

\begin{document}
\bstctlcite{IEEEexample:BSTcontrol}

\title{Where-to-Learn: Analytical Policy Gradient Directed Exploration \\ for On-Policy Robotic Reinforcement Learning}

\author{
    Leixin Chang, 
    Xinchen Yao,
    Ben Liu,
    Liangjing Yang,
    Hua Chen
}


\maketitle

\begin{abstract}
On-policy reinforcement learning (RL) algorithms have demonstrated great potential in robotic control, where effective exploration is crucial for efficient and high-quality policy learning. However, how to encourage the agent to explore the better trajectories efficiently remains a challenge. 
Most existing methods incentivize exploration by maximizing the policy entropy or encouraging novel state visiting regardless of the potential state value. We propose a new form of directed exploration that uses analytical policy gradients from a differentiable dynamics model to inject task-aware, physics-guided guidance, thereby steering the agent towards high-reward regions for accelerated and more effective policy learning.
We integrate our exploration approach into a widely used on-policy RL algorithm, Proximal Policy Optimization, to test and demonstrate its effectiveness. We conduct extensive benchmark experiments and demonstrate the effectiveness of the proposed exploration augmentation method. We further
test our approach on a 6-DOF point-foot robot for velocity tracking locomotion, and conduct the simulation test
and implement a successful sim-to-real deployment as the ultimate
validation. 
Project Page: \texttt{\url{wheretolearn.github.io}}.
\end{abstract}

\begin{IEEEkeywords}
Legged robot, reinforcement learning, model learning for control.
\end{IEEEkeywords}

\section{Introduction} \label{sec:introduction}
\IEEEPARstart{O}{n-policy} Reinforcement Learning (RL) has demonstrated great potential in various robotic control problems in recent years. \cite{hwangbo2019learning, kober2013reinforcement, rudin2022learning, peng2018deepmimic}. RL holds the promise of synthesizing a complex robotic controller by leveraging experiential data from robot-environment interactions, guided by objectives defined as task rewards. 
Policy updates heavily depend on the quality of collected trajectories during exploration, where high-return samples provide positive gradients that reinforce desirable behaviors, while low-return samples suppress undesirable ones. Consequently, the effectiveness of policy improvement relies on whether exploration produces trajectories that are sufficiently informative, highlighting the importance of exploration.
Current on-policy model-free RL algorithms, such as Proximal Policy Optimization (PPO) \cite{schulman2017proximal}, are quite data-expensive to train for the high-dimensional control problem. 
One of the reasons for the low sample efficiency is that the exploration for data basically relies on the undirected, random perturbations derived from the policy's inherent stochasticity \cite{williams1992simple}, entropy maximization \cite{haarnoja2018soft} and noise injection \cite{plappert2017parameter}.
This Brownian-motion style of exploration is agnostic to the environment's underlying physics, making the discovery of high-reward regions a slow and inefficient process. 
Model-based Reinforcement Learning improves sample efficiency by generating synthetic trajectories from a learned dynamics model \cite{kurutach2018model,yu2021combo,feinberg2018model}, but suffers from compounding errors and distributional mismatch \cite{deisenroth2011pilco}.

In this work, we try to implement more sample-efficient and stable RL training by injecting dynamics priors into the exploration process, making the random exploration more directed. 
Specifically, rather than directly using the dynamics model to generate the trajectory data, we make use of the analytical policy gradient propagated from a short-term return objective through the differentiable dynamics model, producing an exploratory policy with short-horizon foresight. The exploratory policy is rolled out to generate more informative and directional trajectory data that serve as the augmentation of the data collected by the primary policy in a standard on-policy algorithm like PPO. Finally, the policy is trained on the augmented exploration data with the PPO stable training mechanism. 
\begin{figure}[t]
    \centering
    \includegraphics[width=0.9\linewidth]{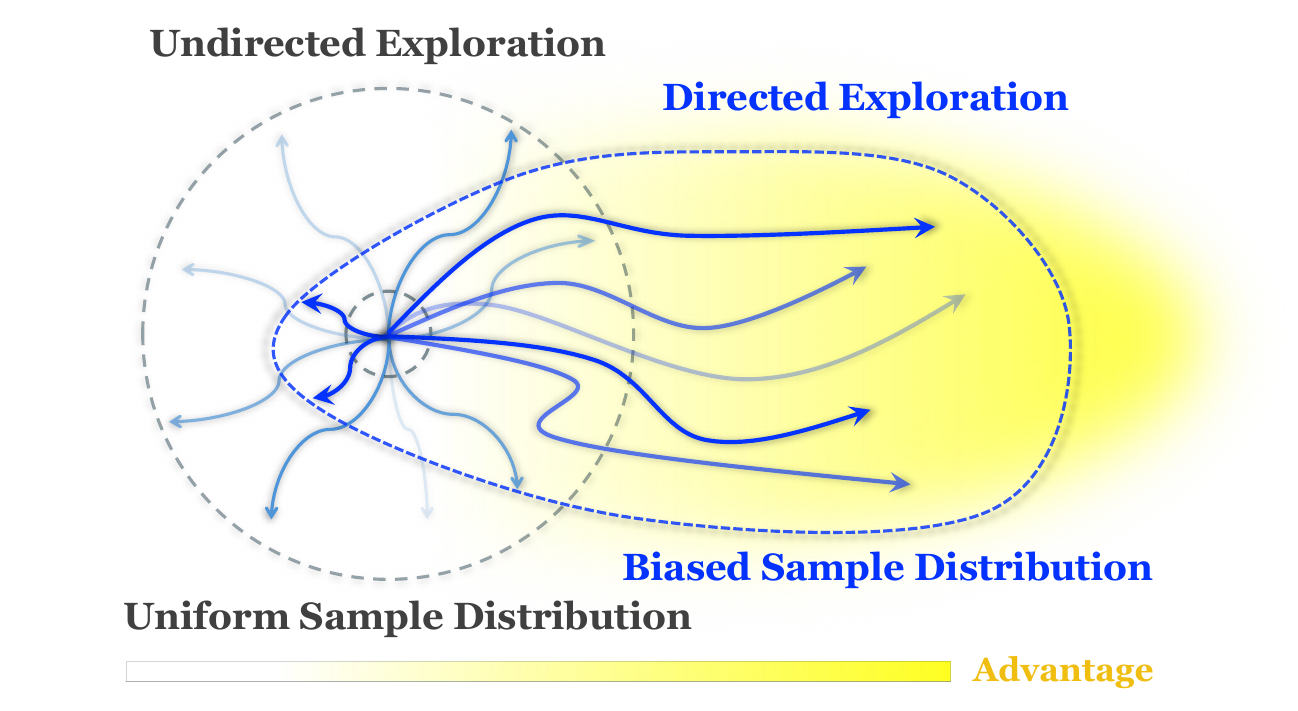}
    \caption{Illustration of the proposed directed exploration.}
    \label{fig:idea_illustration}
    \vspace{-10pt}
\end{figure}

The main contribution of this work is a novel exploration augmentation method for on-policy RL, which leverages the differentiability of dynamics as a kind of inductive bias or dynamics priors to guide exploration and thus enhance the sample efficiency of the learning process.
Further, we theoretically justify our method's mechanism under ideal premises. Specifically, we prove that exploratory policy constitutes a policy improvement and that the exploratory data yield a consistently positive learning signal. 
Finally, comprehensive empirical experiments, including benchmark tests and sim-to-real experiments on a biped robot, are conducted to demonstrate the effectiveness of our method in exploration augmentation and improving sample efficiency over the PPO algorithm.  

\section{Related Work}
\subsection{Exploration in Reinforcement Learning} \label{sec:rl_explore}
In on-policy RL, policy improvement relies on generating sufficiently informative trajectories. Existing methods enhance exploration by rewarding policy stochasticity, e.g., Maximum Entropy \cite{ziebart2008maximum, haarnoja2018soft}, 
or employing task-agnostic intrinsic rewards based on novelty \cite{burdaexploration, pathak2017curiosity, houthooft2016vime}. 
Fundamentally differing from the prior work, our approach derives exploration signals directly from the task objective and system dynamics.
Leveraging differentiable simulation, we inject dynamics priors to provide physics-informed guidance, rendering exploration efficient and task-oriented.
This distinguishes our method by prioritizing exploration \textit{direction} over \textit{breadth}, providing a purposeful \textit{thrust} toward high-reward regions and rendering the learning process significantly more efficient and task-oriented.

\subsection{Policy Learning via Differentiable Dynamics} 
Differentiable dynamics allow backpropagating analytical gradients from the objective directly to policy parameters \cite{freeman1brax, mozer2013focused, metz2021gradients}, enabling the Analytical Policy Gradient (APG) method.
Previous works applied the APG to quadruped locomotion control \cite{songlearning,luo2024residual,schwarke2024learning} or vision-based aerial robot control \cite{zhang2025learning}.
Other works integrate APG into actor-critic frameworks \cite{xuaccelerated,clavera2020model}, using value functions to approximate terminal values for stable training.
While APG offers efficient, low-variance gradients \cite{song2024learning, suh2022differentiable}, it suffers from numerical instability due to contact discontinuities and gradient explosion or vanishing over long horizons \cite{song2024learning,suh2022differentiable}.
Rather than directly optimizing parameters, we use APG to guide exploration. This yields informative trajectories while circumventing the stability issues of direct analytical optimization.

\subsection{Sample-Efficient Reinforcement Learning} 
Sample efficiency is a critical bottleneck in robotic RL due to high physical interaction costs. Existing approaches enhance efficiency via off-policy data reuse or model-based priors. Off-policy algorithms like SAC\mbox{\cite{haarnoja2018soft}} and TD3 \mbox{\cite{fujimoto2018addressing}} utilize replay buffers and double-Q learning to maximize data utility and ensure stability in continuous action spaces. AWR \mbox{\cite{peng2019advantage}} further improves stability by framing policy updates as a weighted supervised learning problem, making it highly effective for learning from demonstrations. Concurrently, model-based RL (MBRL) leverages learned dynamics to generate synthetic trajectories, expand value estimation horizons \mbox{\cite{sutton1991dyna, kuvayev1996model,janner2019trust,kurutach2018model,yu2021combo,feinberg2018model}}, or perform MPC-style planning \mbox{\cite{piche2018probabilistic,chua2018deep}}. However, traditional MBRL often treats dynamics as a black box. In contrast, we leverage analytical gradients from differentiable dynamics to provide physics-informed exploration guidance. This enables directed discovery of high-reward regions, combining on-policy stability with the rapid search capabilities of first-order dynamics priors.

\section{Preliminary} \label{sec:preliminary}
\subsection{Reinforcement Learning for Robotic Control}
Through the lens of reinforcement learning, the robotic control problem is commonly modeled as a Markov Decision Process (MDP), defined by a tuple $\left(\mathcal{S}, \mathcal{A}, \mathcal{T}, \mathcal{R}, \mathcal{\gamma}\right)$, where $\mathcal{S}$ denotes the state space, $\mathcal{A}$ represents the action space, $\mathcal{T}:\mathcal{S} \times \mathcal{A} \rightarrow \mathcal{S}$ is the transition kernel representing the environment dynamics $p(\bd{s}_{t+1}|\bd{s}_t,\bd{a}_t)$, $\mathcal{R}:\mathcal{S} \times \mathcal{A} \times \mathcal{S} \rightarrow \mathbb{R}$ denotes the reward function that assigns a scalar reward to each state transition, quantifying the immediate benefit of taking action $\bd{a}_t$ in state $\bd{s}_t$ under transition dynamics described by $\mathcal{T}$. The agent learns a policy $\pi_{\bd{\theta}}$ to maximizes the discounted cumulative reward $\mathbb{E}_{\tau \sim \pi_{\bd{\theta}}}\left[ \sum_{t}\gamma^t r_t\right]$, where $\tau$ denotes a sampled trajectory under policy $\pi_{\bd{\theta}}$, $r_t$ represents the immediate reward at time $t$ and $\gamma \in [0,1)$ is the discounted factor. 
As mentioned in \secref{sec:rl_explore}, exploration plays a crucial role in effective policy learning. 
Typically, on-policy RL exploration mainly comes from policy stochasticity and entropy maximization.
Our method augments the exploration by leveraging physics priors contained in the environment dynamics, aiming to achieve more effective and task-oriented exploration.
   
\subsection{Policy Learning with Analytical Policy Gradient}
Conceptually, the environment dynamics can be treated as an abstract function $\bd{s}_{t+1}=\bd{f}(\bd{s}_t,\bd{a}_t)$ representing the mapping $\mathcal{F}:\mathcal{S}\times\mathcal{A}\rightarrow \mathcal{S}$, where $\bd{s}_t$ is the system state and $\bd{a}_t$ is the control input at time step $t$. Notably, here the dynamics mapping $\mathcal{F}$ is deterministic and can be regarded as a degenerate form of the stochastic transition kernel $\mathcal{T}$ in the MDP formulation of RL.
And the environment receives a reward signal $r(\bd{s}_t,\bd{a}_t)$ at time step $t$. 
The control policy can be represented by a neural network $\pi_{\bd{\theta}}(\bd{s})$ that takes $\bd{s}_t$ as input and outputs action $\bd{a}_t$. 
If $\bd{f}$ is differentiable w.r.t. $\bd{s}_t$ and $\bd{a}_t$, the system has differentiable dynamics, enabling gradient propagation for policy learning, as explored in APG \cite{freeman1brax, clavera2020model} and First-order Gradient (FoG) \cite{qiao2021efficient}.
\begin{equation}
    \mathcal{L}_{\bd{\theta}}^{\text{APG}} = -\sum^{h-1}_{t=0} r(\bd{s}_t,\bd{a}_t) =-\sum^{h-1}_{t=0} r(\bd{s}_t,\pi_{\bd{\theta}} (\bd{s}_t)),
\end{equation}
where $h$ refers to the trajectory horizon length.
Following \textit{Backpropagation-Through-Time} (BPTT) technique \cite{metz2021gradients}, the gradient of $\mathcal{L}_{\bd{\theta}}^{\text{APG}}$ in terms of policy parameters $\bd{\theta}$ can be expressed as \eqref{eq:bptt},
\begin{equation} \label{eq:bptt}
    \nabla_{\bd{\theta}}\mathcal{L}_{\bd{\theta}}^{\text{APG}}\!=\!-\frac{1}{h}\sum_{t=0}^{h}\left[\frac{\partial r_t}{\partial \bd{a}_t}\frac{\mathrm{d}\bd{a}_t}{\mathrm{d}{\bd{\theta}}}\!+\!\sum_{k=1}^t\!\frac{\partial r_t}{\partial \bd{s}_t}\!\left(\prod_{i=k}^t\frac{\partial \bd{s}_i}{\partial \bd{s}_{i-1}}\!\right)\!\frac{\partial \bd{s}_k}{\partial{\bd{\theta}}}\right],
\end{equation}
where the matrix of partial derivatives $\frac{\partial \bd{s}_i}{\partial \bd{s}_{i-1}}$ is the Jacobian of differentiable dynamics $\bd{f}$. 
Thus, assuming differentiable dynamics and rewards, the policy gradient is computed analytically via BPTT as shown in \eqref{eq:bptt}. 

However, the APG method is challenged by the presence of the noisy optimization landscape, gradient exploding and vanishing and 
empirical bias of gradient computing in the contact-rich environment, which results in unstable training \cite{xuaccelerated, suh2022differentiable, zhong2022differentiable}. 
In our method, APG is repurposed from a policy optimizer into an exploration-guidance mechanism providing informative trajectories for model-free policy optimization. 

\section{Methodology} \label{sec:method}
We aim to augment the undirected exploration of on-policy RL to be a directed and task-oriented exploration by changing the data distribution used for policy network updates.
By utilizing directional gradients from environment dynamics to guide exploration toward high-reward regions, our method accelerates policy learning while circumventing the noisy landscapes and gradient instability inherent to APG. 

\subsection{Analytical Policy Gradient Augmented Exploration}
In each iteration, the exploratory policy is initialized from primary policy and undergoes APG update. 
We then collect trajectories using both policies in parallel environments and merge them into an augmented dataset.
For the update of the policy network and the critic network, we update the critic network with only the data collected by the primary policy and use the aforementioned augmented trajectory dataset to update the policy network under the PPO update rule. The complete pipeline is illustrated in \figref{fig:method_overview}. 
\begin{figure}[htbp]
    \centering
    \includegraphics[width=\linewidth]{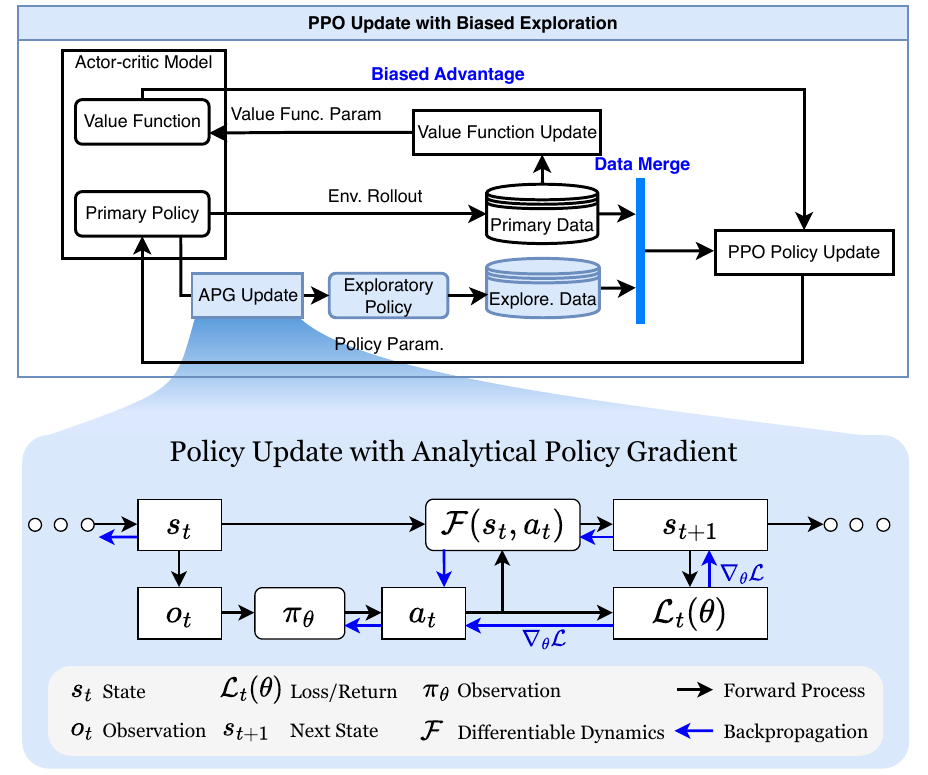}
    \caption{Method Overview.}
    \label{fig:method_overview}
    \vspace{-10pt}
\end{figure}
\subsubsection{Policy Update}
We use the APG method SHAC \cite{xuaccelerated}\footnote[1]{We refer to SHAC as APG in the following to emphasize our contribution lies in the general paradigm of exploration augmentation.} to update the current primary policy to obtain a temporary exploratory policy. 
The exploratory policy collects trajectories to augment the original PPO data. 
To make the experience data more informative in terms of task and dynamics, this exploratory data is merged with the trajectory data collected by the primary PPO policy itself, forming a richer and dynamics-informed dataset for training. 
\begin{figure}[t]
    \centering
    \includegraphics[width=\linewidth]{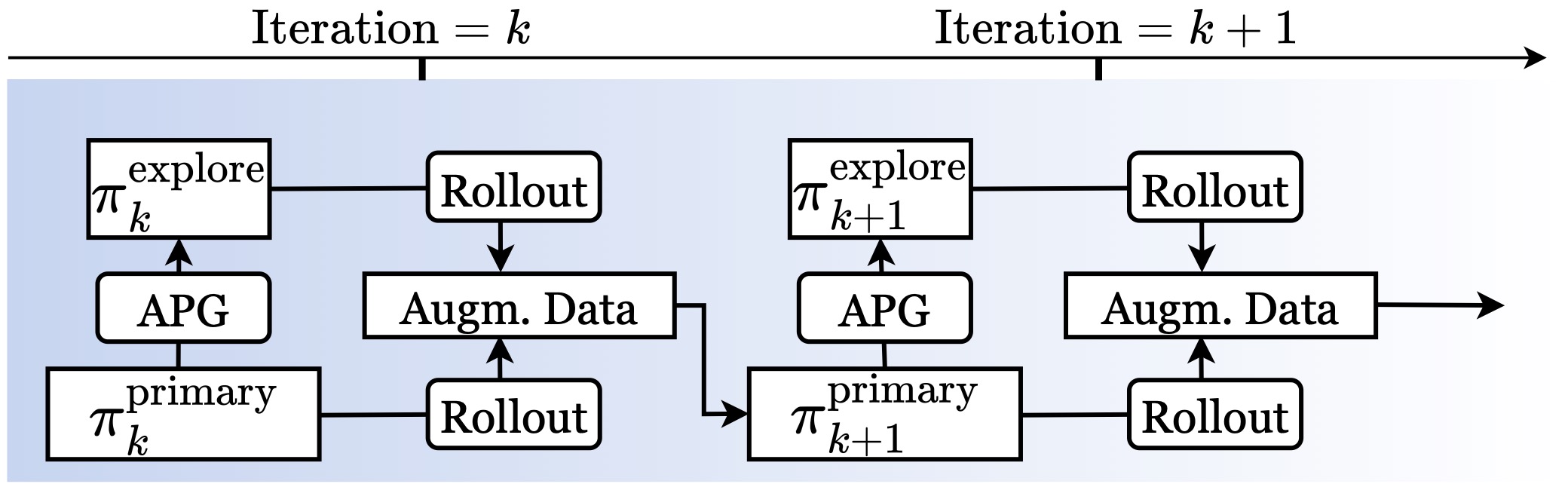}
    \caption{Policy update iteration of the proposed method. In every iteration, the exploratory policy $\pi_k^{\text{explore}}$ is discarded after exploratory data collection.}
    \label{fig:policy_iteration}
    \vspace{-10pt}
\end{figure}
Crucially, high-reward regions discovered by the exploratory policy yield large positive advantages against the primary value baseline. 
This creates a strong learning signal that guides the agent to approach a more promising region in the state-action space. In addition, we use the PPO algorithm to update the policy in the augmented dataset to keep the training stable by leveraging the proximity property of the PPO update. 
\begin{algorithm}[ht]
\footnotesize
    \caption{PPO with APG Directed Exploration}
    \label{alg:apg_ppo}
\begin{algorithmic}[1]
    \State Input: initial policy parameter ${\bd{\theta}}_0^\text{primary}$, value function parameter $\bd{\phi}_0^\text{primary}$, number of iteration $K$.
    \For{$k=0,1,2,\cdots, K$}
        \State Initialize $\pi_{{\bd{\theta}}_k^\text{explore}}$ from $\pi_{{\bd{\theta}}_k^\text{primary}}$
        \If{$k \bmod f=0$} 
        \For{$i=0,1,\cdots,e-1$} \Comment{APG Epoch}
            \State Rollout $\pi_{{\bd{\theta}}_k^\text{explore}}$ for $h$ steps with $N$ parallel agents.
            \State Query $V_{\phi_k}$ for terminal value to construct loss \eqref{eq:shac_loss}.
            \State Update ${\bd{\theta}}_k^\text{explore}$ with \eqref{eq:shac_loss}.
        \EndFor
        
        \State Rollout $\pi_{{\bd{\theta}}_k^\text{primary}}$ with $(1-\alpha)\cdot n $ environments for $\mathcal{D}_k^\text{primary}$. 
        \State Rollout $\pi_{{\bd{\theta}}_k^{\text{explore}}}$ with $\alpha\cdot n$ environments for $\mathcal{D}_k^{\text{explore}}$. 
        \State Merge $\mathcal{D}_k^\text{primary}$ and $\mathcal{D}_k^{\text{explore}}$ to form $\mathcal{D}_k^{\text{aug}}$.
        \State Compute advantage $\hat{A}_t$ based on $V_{\phi_k}^{\pi^{\text{primary}}}$, $\mathcal{D}_k^{\text{aug}}$.
        \State Update $\pi_{{\bd{\theta}}_k^\text{primary}}$ with $\mathcal{D}_k^{\text{aug}}$ by PPO loss \eqref{eq:loss_ppo}.
        \State Update $V_{\phi_k}^{\pi^{\text{primary}}}$ with $\mathcal{D}_k^\text{primary}$ and \eqref{eq:loss_value}.
        \State Discard $\pi_{{\bd{\theta}}_k^\text{explore}}$.
        \Else
        \State Perform standard PPO to update $\pi_{{\bd{\theta}}_k^\text{primary}}$.
        \EndIf
    \EndFor
    \State Output: policy parameter ${\bd{\theta}}_K^\text{primary}$, value function parameter $\bd{\phi}_K^\text{primary}$.
\end{algorithmic}
\end{algorithm}
In every policy update iteration, as illustrated in the \figref{fig:policy_iteration}, the policy parameter initialized from ${\bd{\theta}}_k^\text{primary}$ is updated with APG to yield a temporary exploratory policy $\pi_{{\bd{\theta}}_k^\text{explore}}$, where the APG update process is outlined in the bottom frame of \figref{fig:method_overview}. Following the SHAC method{\color{blue} }\cite{xuaccelerated}, the initialized exploratory policy $\pi_{{\bd{\theta}}_k^{\text{primary}}}$ is rolled out to collect $N$ sample trajectories with a horizon length $h$ parallel in the simulation. 
The terminal value of each trajectory is estimated by the critic network, forming an infinite-horizon objective to alleviate the local minima problem of the analytical gradients \cite{xuaccelerated}. 
The optimization objective for the exploratory policy can be formulated as \eqref{eq:shac_loss} following \cite{xuaccelerated}, where $\gamma$ is the discount factor. Note that the right part of \eqref{eq:shac_loss} is the state-action value function, a.k.a. the Q function. Following SHAC \cite{xuaccelerated} and BPTT \cite{metz2021gradients}, we further derive the policy gradient of $\mathcal{L}_{\bd{\theta}}^{\text{explore}}$ shown as \eqref{eq:shac_grad}, with \eqref{eq:shac_grad_cond} holding when $t_0\leq t<t_0+h$ . We use the differentiable physics engine \texttt{Brax} \cite{freeman1brax} to implement the gradient computation with JAX autograd function. Specifically, the exploratory policy is modeled as a Gaussian policy as the primary policy. Thus, the reparameterization sampling method is employed for the stochastic policy to enable the computation of $\frac{\partial \pi_{\bd{\theta}}(\bd{s})}{\partial \bd{\theta}}$ and $\frac{\partial \pi_{\bd{\theta}}(\bd{s})}{\partial \bd{s}_t}$.
\begin{equation}
        \mathcal{L}_{\bd{\theta}}^{\text{explore}} \!=\! -\frac{1}{Nh} \sum^{N}_{i=1}\left[ \left(\sum^{t_0 + h-1}_{t=t_0} \gamma^{t-t_0} r(\bd{a}_t^i,\bd{s}_t^i)\right)\! +\! \gamma^{h}V_{\phi}(\bd{s}_{t+h}^i) \right]
 \label{eq:shac_loss} 
\end{equation}
\begin{equation}
\nabla_{\bd{\theta}}\mathcal{L}_{\bd{\theta}}^{\text{explore}}=\sum_{i=1}^{N}\sum_{t=t_0}^{t_0+h-1} 
\Bigg(\frac{\partial\mathcal{L}_{\bd{\theta}}^{\text{explore}}}{\partial \bd{a}_t^{i}}\Bigg)
\Bigg(\frac{\partial \pi_{\bd{\theta}}(\bd{s}_t^i)}{\partial \bd{\theta}} \Bigg)
 \label{eq:shac_grad} 
\end{equation}

\begin{equation} 
\label{eq:shac_grad_cond} 
\left\{
\begin{aligned} 
\frac{\partial\mathcal{L}_{\bd{\theta}}^{\text{explore}}}{\partial \bd{a}_t^{i}}&= -\gamma^{t-t_0}\frac{1}{Nh}\frac{r(\bd{a}_t^i,\bd{s}_t^i)}{\partial \bd{a}_t^i} + \frac{\partial \mathcal{L}_{\bd{\theta}}^{\text{explore}}}{\partial \bd{s}_{t+1}^i}  \frac{\partial \bd{f}}{\partial \bd{a}_t^i }  \\
\frac{\partial\mathcal{L}_{\bd{\theta}}^{\text{explore}}}{\partial \bd{s}_t^{i}}&= -\gamma^{t-t_0}\frac{1}{Nh}\frac{r(\bd{a}_t^i,\bd{s}_t^i)}{\partial \bd{s}_t^i} +  \\
&\quad \quad \Bigg(\frac{\partial \mathcal{L}_{\bd{\theta}}^{\text{explore}}}{\partial \bd{s}_{t+1}^i} \Bigg) 
\Bigg(
\frac{\partial \bd{f}}{\partial \bd{s}_t^i} + 
\frac{\partial \bd{f}}{\partial \bd{a}_t^i}\frac{\partial \mathcal{\pi_{\bd{\theta}}}}{\bd{s}_t^i}
\Bigg) \\
\frac{\partial\mathcal{L}_{\bd{\theta}}^{\text{explore}}}{\partial \bd{s}_{t_0+h}^i} &= -\gamma^h \frac{1}{Nh} \frac{\partial V_\phi(s_{t_0+h}^i)}{\partial \bd{s}^i_{t_0+h}}.
\end{aligned}   
\right. 
\end{equation}

In parallel, $\pi_{{\bd{\theta}}_k^\text{explore}}$ and $\pi_{{\bd{\theta}}_k^\text{primary}}$ collect trajectory datasets $\mathcal{D}_k^{\text{explore}}$ and $\mathcal{D}_k^\text{primary}$ using $\alpha\cdot n$ and $(1-\alpha)\cdot n$ environments, respectively. These are subsequently merged to form the augmented dataset $\mathcal{D}_k^{\text{aug}}$. 
Then two dataset $\mathcal{D}_k^{\text{explore}}$ and $\mathcal{D}_k^\text{primary}$ are merged to form the augmented dataset $\mathcal{D}_k^{\text{aug}}$. Then, the primary policy parameter is updated on this augmented dataset $\mathcal{D}_k^{\text{aug}}$ with the surrogate loss function \eqref{eq:loss_ppo}, where $\hat{A}$ denotes the advantage estimate at time step $t$, computed using the Generalized Advantage Estimation (GAE) to guide the policy update. Afterward, $\pi_{{\bd{\theta}}_k^\text{explore}}$ is re-initialized from the updated primary policy to prevent APG instability \mbox{\cite{suh2022differentiable}} and excessive distributional shifts.
\begin{equation} \label{eq:loss_ppo}
\begin{aligned}
\mathcal{L}_{\bd{\theta}}=-\frac{1}{\vert \mathcal{D}_k^{\text{aug}} \vert T}
\sum_{\tau\in \mathcal{D}_k^{\text{aug}}} \sum_{t=0}^T 
\left[
\min \left(  
    \eta_t({\bd{\theta}}) \hat{A}_t(\bd{s}_t,\bd{a}_t), 
    \right. 
    \right.\\
    \left.
    \left.
    \text{clip} 
    \left(
        \eta_t({\bd{\theta}}), 1-\epsilon,1+\epsilon
    \right)\hat{A}_t(\bd{s}_t,\bd{a}_t)
\right)
+\beta\mathcal{H}\big(\pi_{\bd{\theta}}(\cdot|\bd{s}_t)\big)
\right]
\end{aligned}
\end{equation}
\begin{equation} \label{eq:ppo_impt_sample}
    \eta_t({\bd{\theta}}) = \left\{ 
            \begin{aligned}
                \frac{\pi_{{\bd{\theta}}}(\bd{a}_t|\bd{s}_t)}{\pi_{{\bd{\theta}}_k^\text{primary}}(\bd{a}_t|\bd{s}_t)}, \quad & \text{if} \ \ (\bd{a}_t, \bd{s}_t) \in \mathcal{D}_k^\text{primary}   \\
                \frac{\pi_{{\bd{\theta}}}(\bd{a}_t|\bd{s}_t)}{\pi_{{\bd{\theta}}_k^{\text{explore}}}(\bd{a}_t|\bd{s}_t)},\quad & \text{if} \ \ (\bd{a}_t, \bd{s}_t) \in \mathcal{D}_k^{\text{explore}}. 
            \end{aligned}
            \right.
\end{equation}
The surrogate loss \eqref{eq:loss_ppo} is constructed on the composite data $\mathcal{D}_k^{\text{aug}}$ collected by both a primary and an exploratory policy. To ensure unbiased policy gradient estimation, we design a piecewise importance sampling ratio $\eta_t(\bd{\theta})$ shown as \eqref{eq:ppo_impt_sample}. In addition, we also incorporate the policy entropy bonus $\mathcal{H}\big(\pi_{\bd{\theta}}(\cdot|\bd{s}_t)\big)$ in \eqref{eq:loss_ppo} following the original PPO algorithm design{\color{blue} }\cite{schulman2017proximal}.
Notably, the reason we retain data from the primary policy rather than relying solely on the exploratory data for policy learning is that APG-guided trajectories, while task-oriented, may lead the policy toward suboptimal behaviors and suffer from training instability induced by low-quality gradients of differentiable dynamics implementation{\color{blue} }\cite{suh2022differentiable}. The primary data here $\mathcal{D}_k^{\text{primary}}$ complements this by preserving broader exploration.

\subsubsection{Critic Update}
For value estimation, the value function $V_{\phi_k}^{\pi^{\text{primary}}}$ is trained by minimizing the mean squared error against the estimated returns $\hat{R}_t$, as defined in \eqref{eq:loss_value}.
\begin{equation} \label{eq:loss_value}
    \mathcal{L}_{\phi}=\frac{1}{|\mathcal{D}_k^{\text{primary}}|T}\sum_{\tau\in\mathcal{D}_k^{\text{primary}}}\sum_{t=0}^T\left(V_{\phi_k}^{\pi^{\text{primary}}}(s_t)-\hat{R}_t\right)^2
    \vspace{-10pt}
\end{equation}
Notably, value function $V_{\phi_k}^{\pi^{\text{primary}}}$ fitting is on the data collected on the primary policy $\pi_{\bd{\theta}_k^{\text{primary}}}$ rather than $\pi_{\bd{\theta}_k^{\text{explore}}}$, which means $V_{\phi_k}^{\pi^{\text{primary}}}$ only depends on $\pi_{\bd{\theta}_k^{\text{primary}}}$. Consequently, $V_{\phi_k}^{\pi^{\text{primary}}}$ could provide an unbiased on-policy value baseline for the advantage estimation. 
We use the APG update frequency $f$ and the number of update epochs $e$ to control the generation and usage of the exploratory policy. Specifically, an APG update is performed $e$ times to obtain the exploratory policy and the exploratory data collection is executed once every $f$ training iterations. 
The detailed pipeline is given in \pseref{alg:apg_ppo} and \figref{fig:method_overview}.

\subsection{Analysis on Mechanism of Exploration Augmentation} 
In this section, we analyze how the proposed method guides the policy learning process through the biased advantage introduced by the analytical gradient. 
\subsubsection{Directed Exploration versus Undirected Exploration} \label{sec:directed_explore} 

Standard on-policy exploration relies on undirected stochasticity or novel state visiting, which promotes broad coverage but is often sample-inefficient in high-dimensional spaces due to its disregard for potential state values \cite{schulman2017proximal, ziebart2008maximum, haarnoja2018soft, burdaexploration}.
In contrast, we generate a directed exploratory policy, $\pi_{\bd{\theta}^{\text{explore}}}$, by directly backpropagating the analytical gradient $\nabla_{\bd{\theta}}\mathcal{L}_{\bd{\theta}}^{\text{explore}}$ through the differentiable dynamics. 
Since the objective \eqref{eq:shac_loss} estimates the discounted return, the resulting gradient $\nabla_{\bd{\theta}}\mathcal{L}_{\bd{\theta}}^{\text{explore}}$ explicitly points toward the direction of steepest return maximization in the short horizon.
Consequently, $\pi_{\bd{\theta}_k^{\text{explore}}}$ executes directed exploration, collecting trajectories biased toward high-reward regions that the primary policy may not yet cover.
Integrating this informative data $\mathcal{D}^{\text{explore}}$ into the update step provides a high-quality learning signal, thereby accelerating convergence toward optimal policy regions.

\subsubsection{Biased Advantage as Augmented Learning Signal} \label{sec:biased_advantage}
In this section, the theoretical mechanism of how the exploratory data improves policy learning is detailed.
As mentioned in \secref{sec:directed_explore}, the exploratory policy training loss \eqref{eq:shac_loss} is an approximation for the negative normalized discounted return, as shown in \eqref{eq:shac_loss_approx}.
Here, we take the ideally precise form of exploratory policy loss for theoretical analysis. 
In the following, we try to justify that the expectation of advantage computed on the exploratory data $\mathcal{D}^{\text{explore}}$ is equal to or greater than that computed on $\mathcal{D}^{\text{primary}}$.
\begin{equation} \label{eq:shac_loss_approx}
\begin{aligned}
    \mathcal{L}_{\bd{\theta}}^{\text{explore}}\! &=\!-\frac{1}{hN}\sum^{N}_{i=1}\underbrace{\Bigg[ 
        {\left(\sum^{t_0 + h-1}_{t=t_0} \gamma^{t-t_0} r(\bd{a}_t^i,\bd{s}_t^i)\right) \! 
         + \! \gamma^{h}V_{\phi}(\bd{s}_{t+h}^i)}
         \Bigg]}_\text{Discounted Return Approximation} \\
        &\approx -\frac{1}{h} \cdot \mathbb{E}_{\tau \sim \pi_{\bd{\theta}}}\left[ \sum_{t}\gamma^t r_t\right]
\end{aligned}
\vspace{-10pt}
\end{equation}

\begin{lemma}[Local Policy Improvement from an APG Update]
\label{lem:policy_improvement}
Let the performance objective be $J(\bd{\theta})=\mathbb{E}_{\tau \sim \pi_{\bd{\theta}}}\left[ \sum_{t}\gamma^t r_t\right]$.
Assume \(J(\bd{\theta})\) is \(L\)-smooth, and
\[
\bd{\theta}_k^{\text{explore}}
=
\bd{\theta}_k+\eta\nabla_{\bd{\theta}}J(\bd{\theta}_k), \quad 0<\eta<2/L.
\]
Then
$J(\bd{\theta}_k^{\text{explore}})\!>\! J(\bd{\theta}_k)$,
whenever
$\nabla_{\bd{\theta}}J(\bd{\theta}_k)\neq 0$.
\end{lemma}

\begin{proof}
By the ascent counterpart of the descent lemma,
\begin{equation}
J(\bd{\theta}')
\ge
J(\bd{\theta})
+
\nabla_{\bd{\theta}}J(\bd{\theta})^\top(\bd{\theta}'-\bd{\theta})
-
\frac{L}{2}\|\bd{\theta}'-\bd{\theta}\|^2.
\end{equation}
Let
$
\bd{\theta}'=\bd{\theta}_k^{\text{explore}}
=
\bd{\theta}_k+\eta\nabla_{\bd{\theta}}J(\bd{\theta}_k)$.
Then
\begin{equation}
J(\bd{\theta}_k^{\text{explore}})
\ge
J(\bd{\theta}_k)
+
\left(\eta-\frac{L}{2}\eta^2\right)
\|\nabla_{\bd{\theta}}J(\bd{\theta}_k)\|^2.
\end{equation}
Therefore, if \(0\!<\!\eta\!<\!2/L\), we have
\(J(\bd{\theta}_k^{\text{explore}})\!>\! J(\bd{\theta}_k)\)
whenever
\(\nabla_{\bd{\theta}}J(\bd{\theta}_k)\neq 0\).
Hence, \(\pi_{\bd{\theta}_k^{\text{explore}}}\) is a local improvement over \(\pi_{\bd{\theta}_k}\).
\end{proof}

\begin{theorem}[Superiority of the Exploratory Policy in Expected Advantage]
\label{thm:advantage_superiority}
    Let $\pi_k^{\text{primary}}$ be the primary policy and $\pi_k^{\text{explore}}$ be the exploratory policy obtained from Lemma \ref{lem:policy_improvement}. 
    We assume that the value function $V_k$ is a perfect estimator of the true value of the primary policy: $V_k(s) = V^{\pi_k^{\text{primary}}}(s)$. Then
    \[
        \mathbb{E}_{(s,a) \sim \pi_k^{\text{explore}}}[\hat{A}^{\pi_k^{\text{primary}}}(s,a)] >\mathbb{E}_{(s,a) \sim \pi_k^{\text{primary}}}[\hat{A}^{\pi_k^{\text{primary}}}(s,a)]
    \]
\end{theorem}
\begin{proof}
    For data from $\pi^{\text{primary}}_k$, 
    \begin{equation}
    \vspace{-5pt}
        \mathbb{E}_{(s,a) \sim \pi_k^{\text{primary}}}[\hat{A}^{\pi_k^{\text{primary}}}(s,a)]=0.
    \end{equation}
    For data from $\pi^{\text{explore}}_k$, the policy improvement theorem{\color{blue} }\cite{kakade2002approximately,schulman2015trust} gives
    \begin{equation} \label{eq:policy_improve}
    \vspace{-5pt}
        J(\pi_k^{\text{explore}})-J(\pi_k^{\text{primary}})\!=\!\mathbb{E}_{\tau\sim\pi_k^{\text{explore}}}\left[\!\sum_t \gamma^t \hat{A}^{\pi_k^{\text{primary}}}(\bd{s}_t,\bd{a}_t)\!\right] \! > \! 0.
    \end{equation}
    Hence, 
    \begin{equation}
    \vspace{-5pt}
        \mathbb{E}_{(s,a) \sim \pi_k^{\text{explore}}}[\hat{A}^{\pi_k^{\text{primary}}}(s,a)]\!>\! \mathbb{E}_{(s,a) \sim \pi_k^{\text{primary}}}[\hat{A}^{\pi_k^{\text{primary}}}(s,a)] \!=\! 0.
    \end{equation}
    This completes the proof.
\end{proof}
\theoref{thm:advantage_superiority} formally establishes that data collected by the exploratory policy are superior on average to the primary policy's own value baseline, providing the theoretical grounding for our learning mechanism.  The systematic bias forms the source of the augmented learning signal, where the strong gradient from the non-negative advantage directs the policy to increase the action probability in the newly discovered high-return trajectories. 
While \theoref{thm:advantage_superiority} provides the theoretical grounding, the premises may not hold perfectly in practice. 
The analytical gradient quality and thus the consequent policy improvement are contingent on the fidelity of the differentiable simulator \cite{suh2022differentiable}. 
Furthermore, the learned value function for the terminal value estimation in the optimization objective for obtaining the exploratory policy is an imperfect estimator, especially in the early training stage.
Consequently, while the data collected by the exploratory policy provides a powerful guiding signal on the whole, the advantage estimates may not be strictly greater than those from the primary data at every point in the training. 
To handle potentially suboptimal data from a corrupted $\pi_{\bd{\theta}^\text{explore}}$, our framework uses the value function as a safety filter. Such trajectories yield negative advantage estimates via GAE, allowing PPO objective to effectively down-weight or \textit{reject} them, thus shielding the primary policy.

In general, the exploratory policy acts as a \textit{scout} that identifies promising regions, while the primary policy learns from these curated trajectories. This reduces inefficient wandering, making learning focus the learning on demonstrably better behaviors and accelerating convergence.

\section{Experiment Results} \label{sec:experiment}
We aim to answer these two primary questions with our experiments: (1) Can our exploration augmentation mechanism help the training process access the higher-reward region in the state-action space and further improve the learning sample efficiency through that biased exploration? 
(2) Is the proposed method deployable in the real world?
To this end, we evaluate on a series of comparative experiments on benchmark environments implemented in \texttt{MuJoCoPlayground} (MJP) \cite{zakka2025mujoco} and a set of sim-to-real deployment experiments on a biped point-foot robot, LimX TRON. 
Training runs on an RTX 4090 GPU with an AMD EPYC 9354 CPU.
We employed the JAX-based PPO implementation in \texttt{Brax} \cite{freeman1brax}, and we implement our method on top of differentiable dynamics provided by \texttt{Brax} and \texttt{MuJoCo-XLA} (MJX) \cite{mujoco_mjx_docs, todorov2012mujoco}. 
We conducted preliminary hyperparameter searches of our method for each task and report the best settings.

\subsection{Comparative Evaluation with Benchmark Experiment}
\begin{figure*}[htbp]
    \centering        
        \begin{minipage}{\textwidth}
            \centering
            \subfloat{\includegraphics[width=0.22\textwidth]{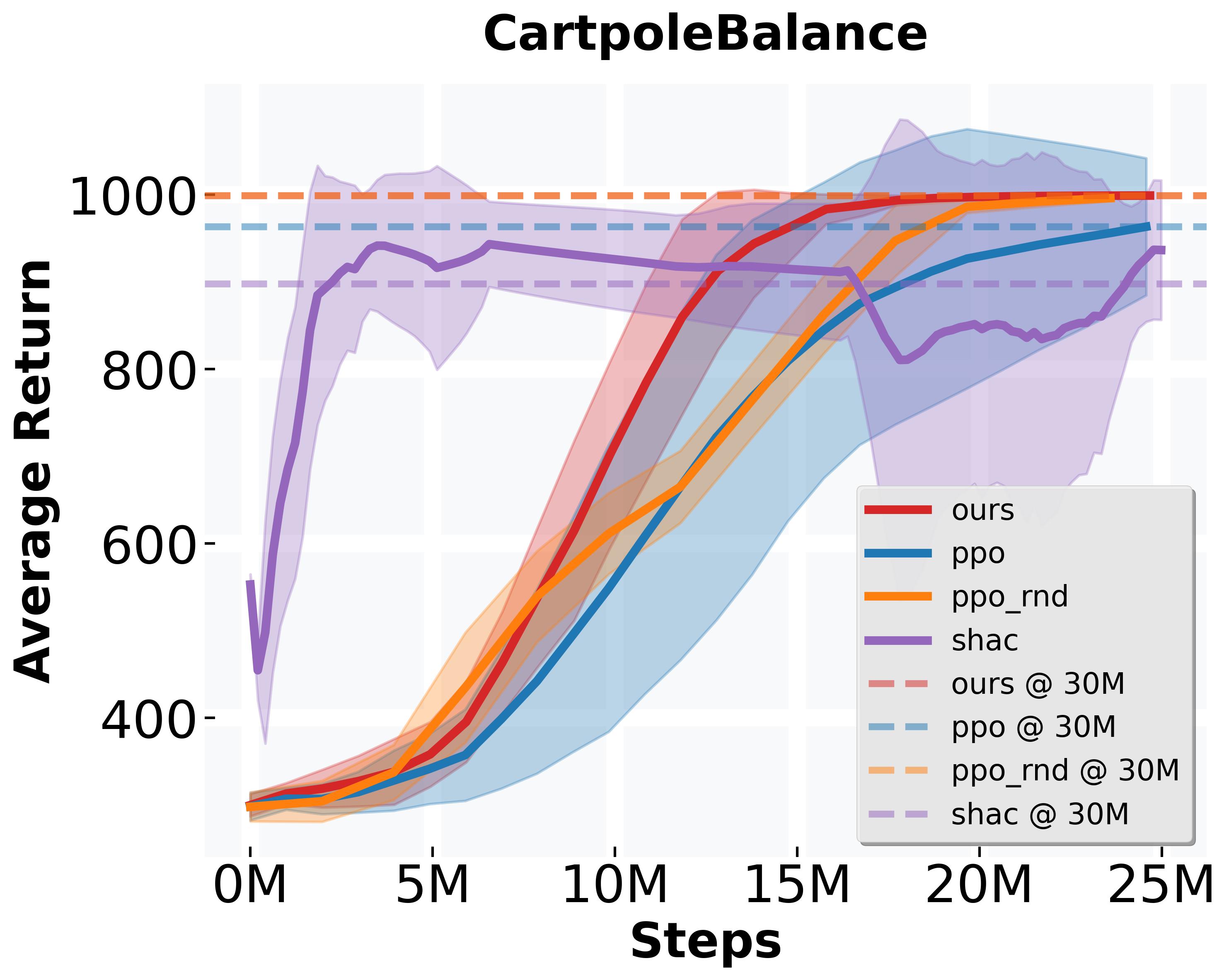}} \hfill
            \subfloat{\includegraphics[width=0.22\textwidth]{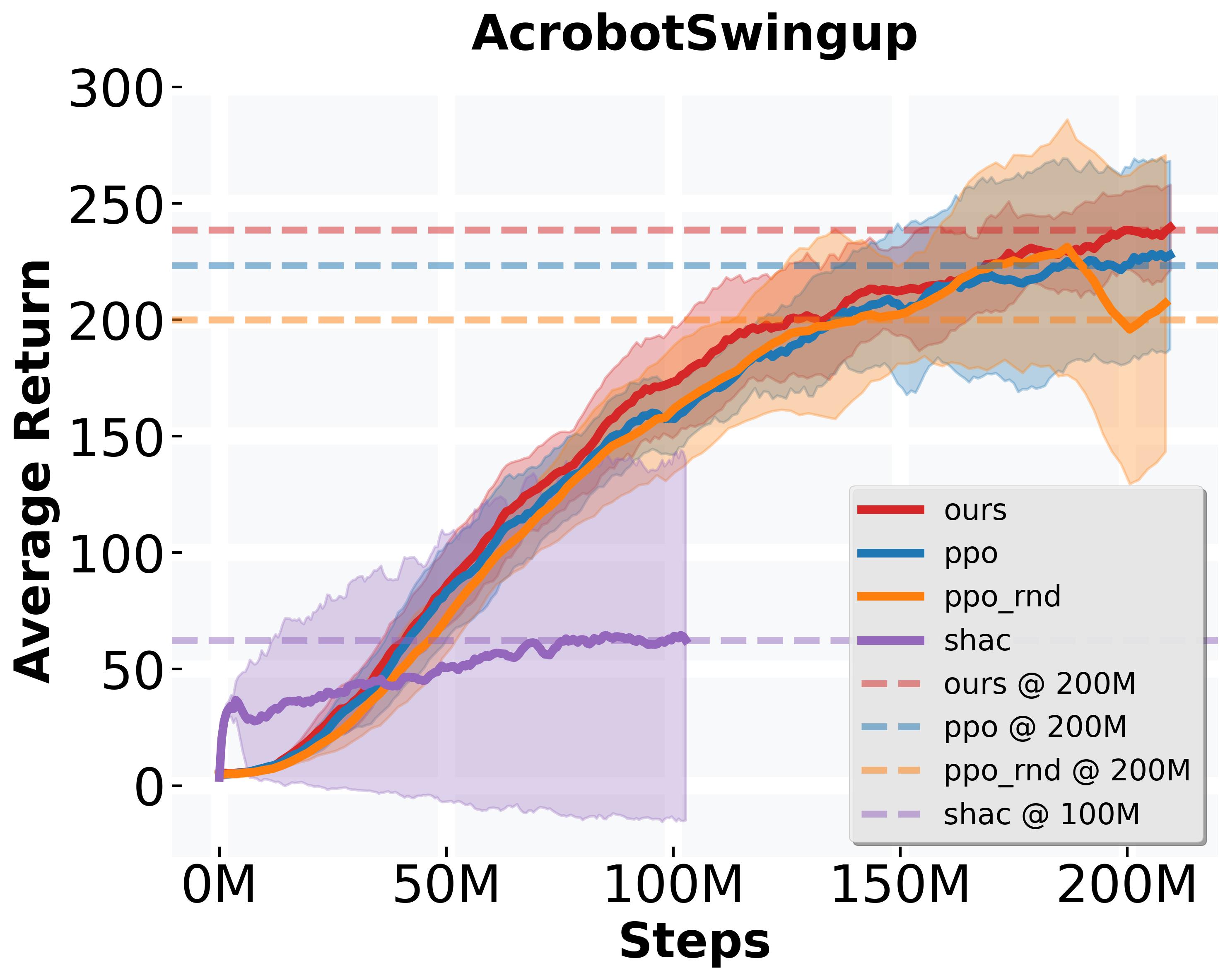}} \hfill
            \subfloat{\includegraphics[width=0.22\textwidth]{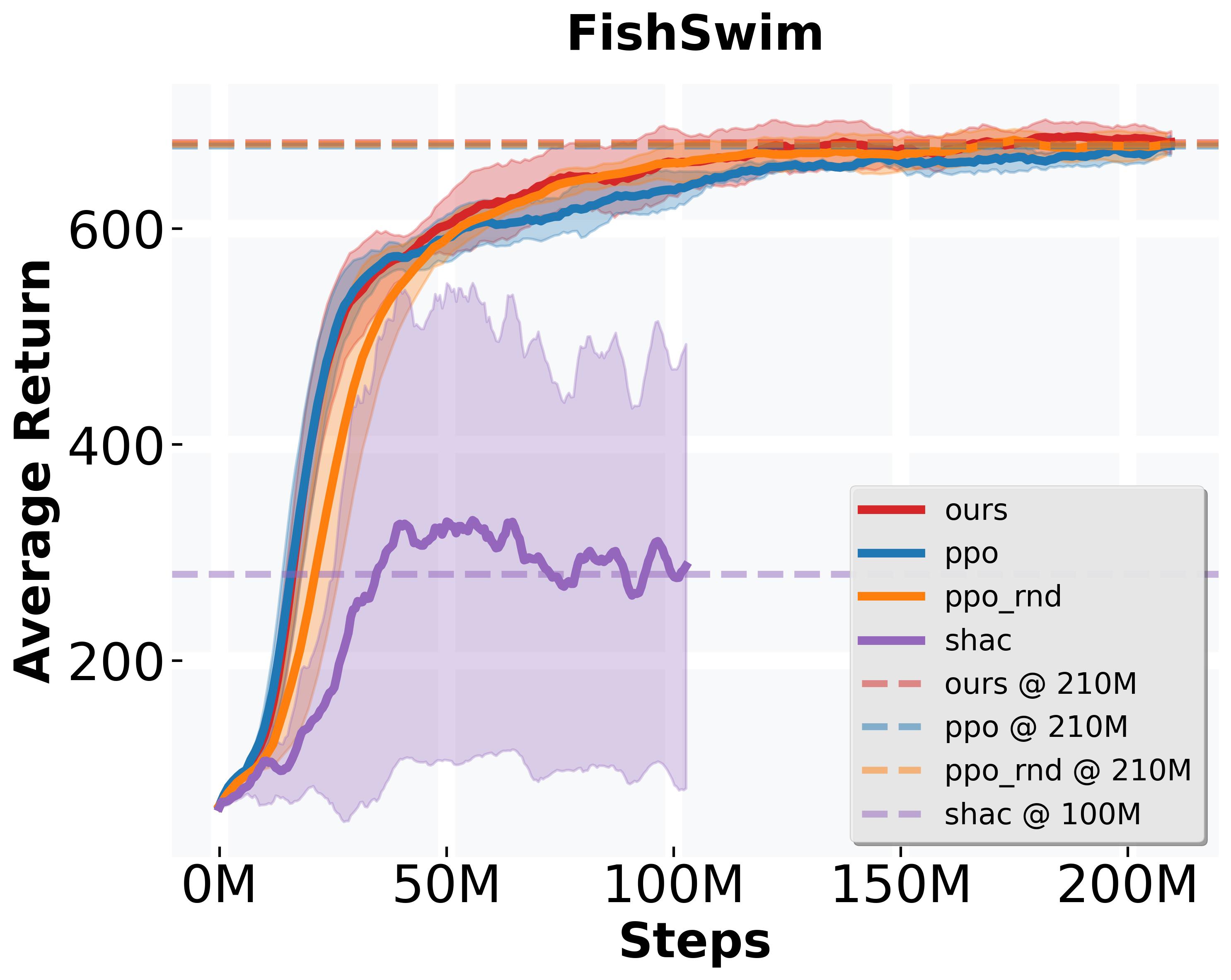}} \hfill
            \subfloat{\includegraphics[width=0.22\textwidth]{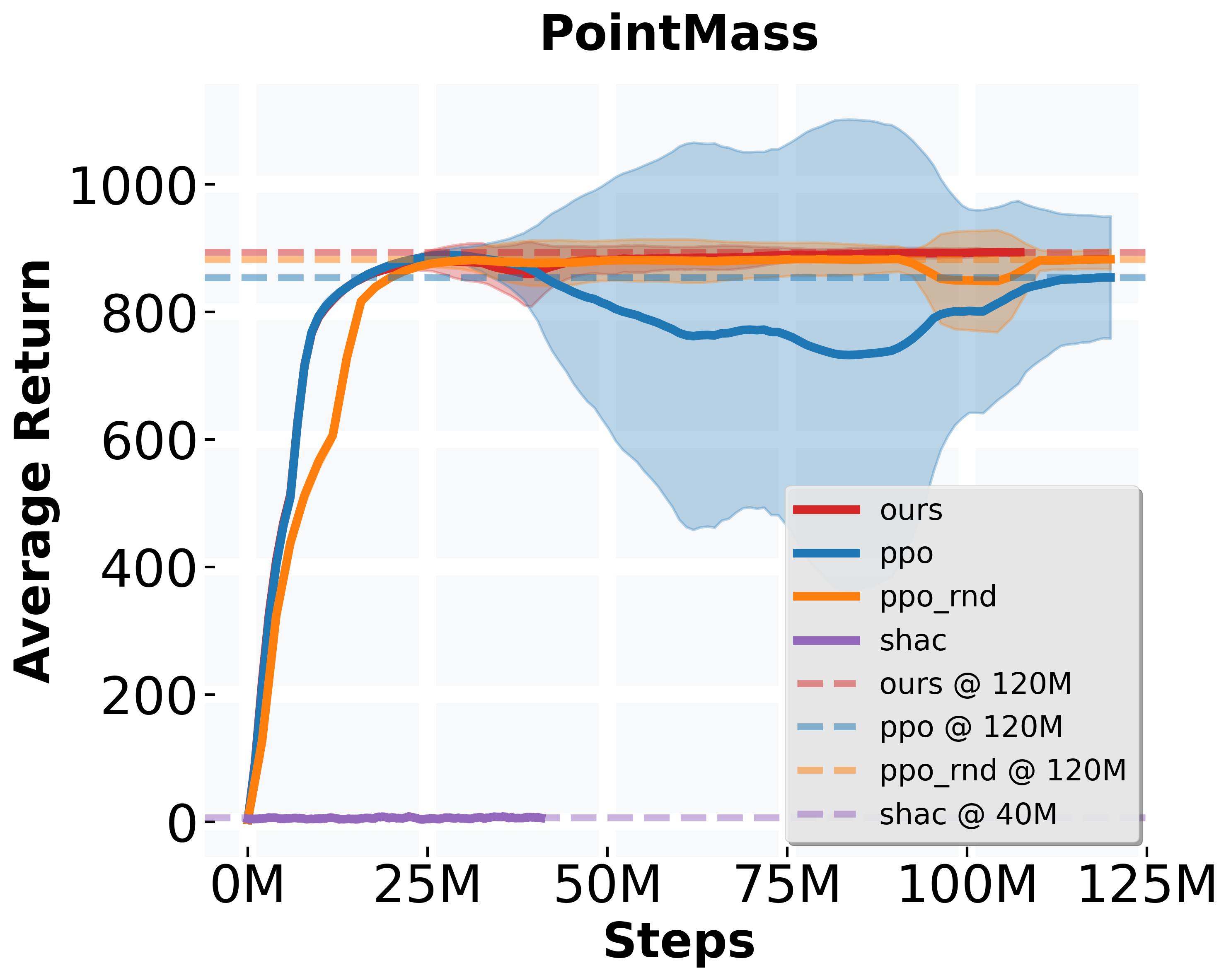}} 
            \\ 
            \subfloat{\includegraphics[width=0.22\textwidth]{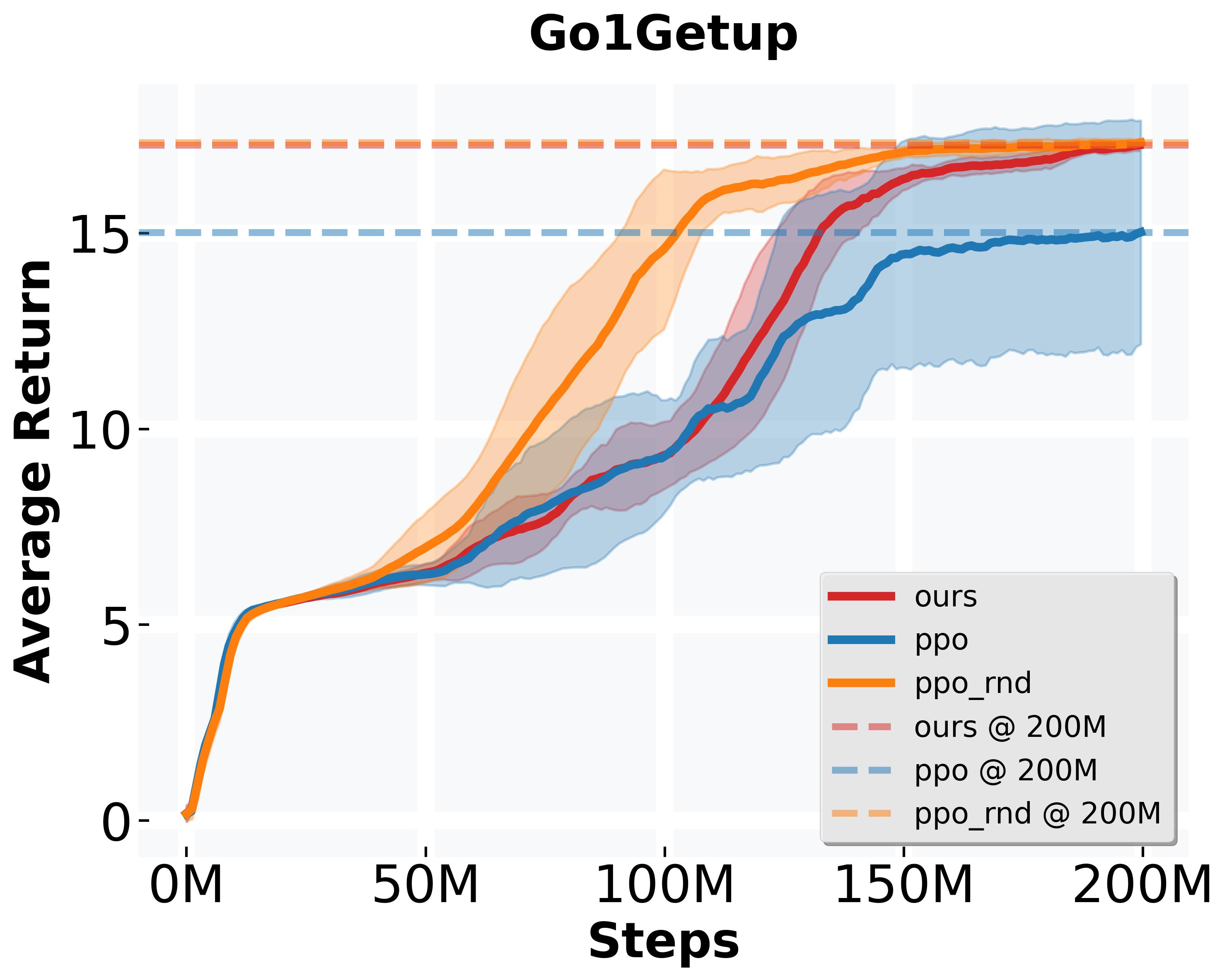}} \hfill
            \subfloat{\includegraphics[width=0.22\textwidth]{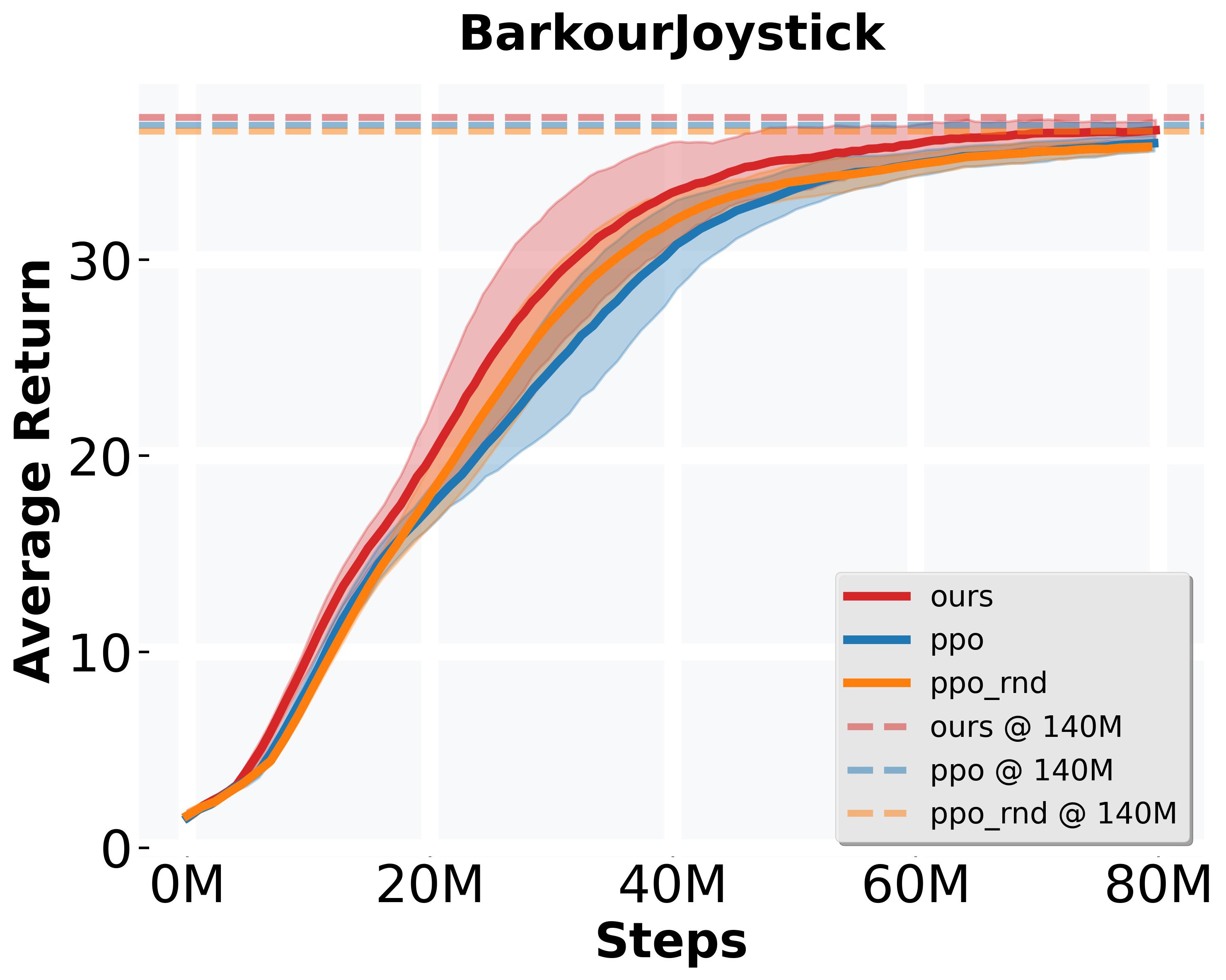}} \hfill
            \subfloat{\includegraphics[width=0.22\textwidth]{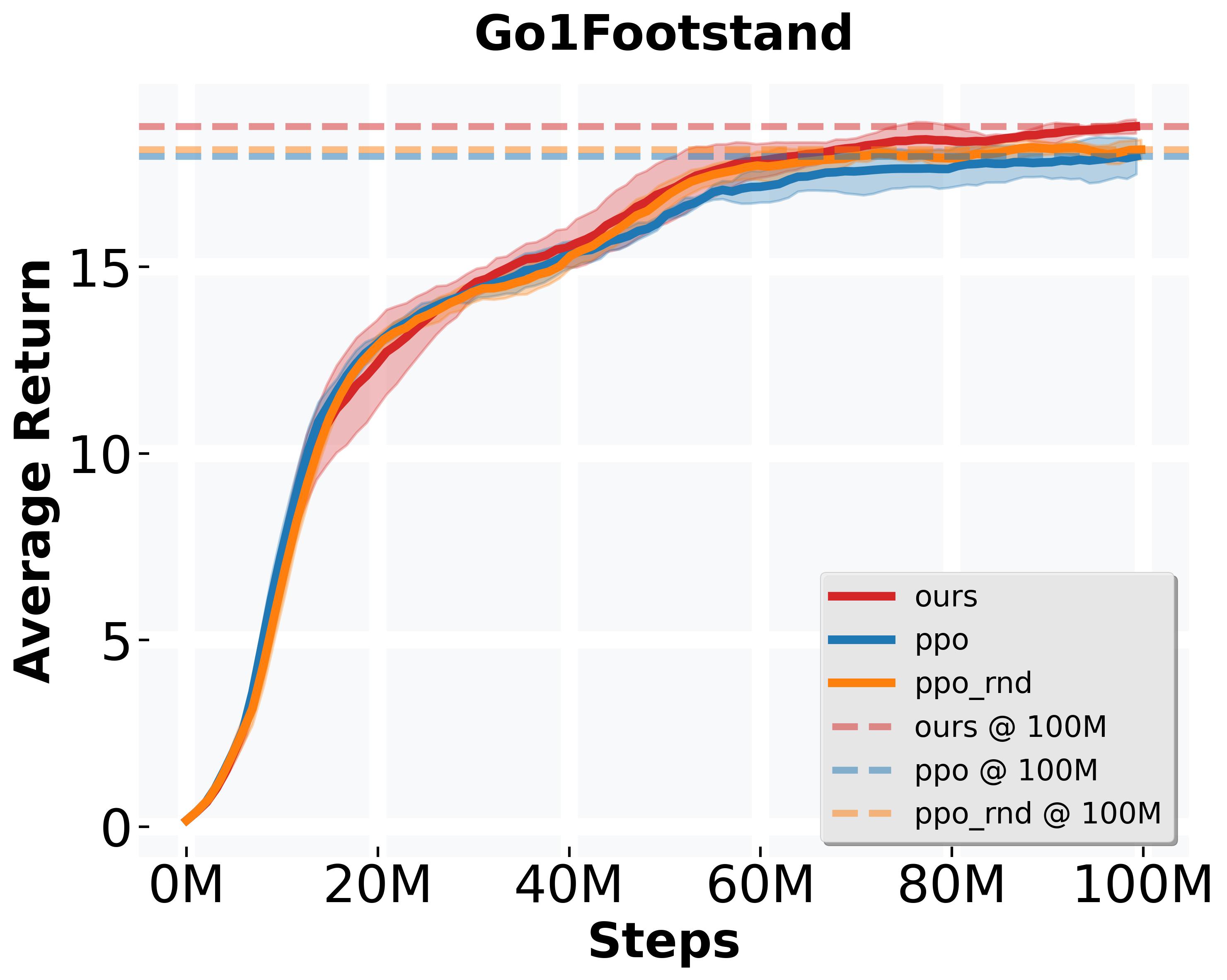}} \hfill
            \subfloat{\includegraphics[width=0.22\textwidth]{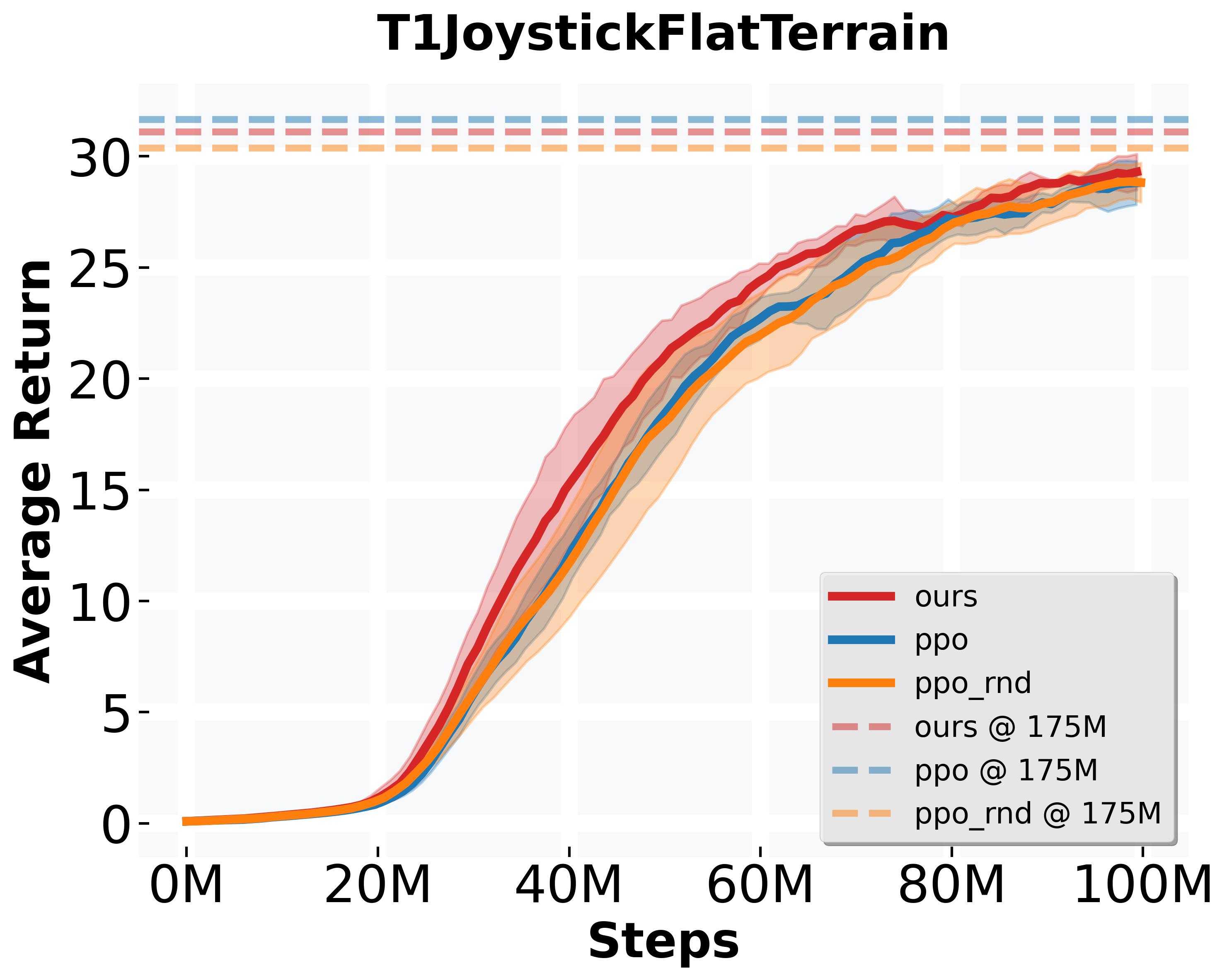}}
        \end{minipage}
    \caption{Training curve of 8 benchmark tasks comparing the proposed method against the PPO baseline, the PPO with RND technique. The curves of SHAC serve as a reference. Solid lines and shaded regions depict the mean and standard deviation among the five trials, respectively. 
    Across most benchmark tasks, our method achieves higher or matched asymptotic performance over the PPO baseline and PPO with RND technique, with improved sample efficiency and better training stability.}
    \label{fig:bcmk_result}
    \vspace{-15pt}
\end{figure*}
We validate the effectiveness of our method against PPO with entropy bonus and PPO with RND exploration on eight benchmarks in MJP: four \texttt{dm\_control\_suite} tasks displayed in the top row of \figref{fig:bcmk_result} and four locomotion tasks shown in the bottom row. And we provide the curve of SHAC method on the four \texttt{dm\_control\_suite} tasks as a reference for the performance of pure APG. 
The PPO part of its combination with RND and our method shares the same settings as baseline PPO, while our method’s hyperparameters are listed in \tabref{tab:benchmark_params}. And the SHAC follows $h$ of 32, $\lambda$ of 0.99, $N$ of 128 and $\gamma$ of 0.95. Reported environment steps include both PPO and APG rollouts.
 
\begin{table}[bp]
    \centering
    \caption{Hyperparameters of our method in the experiments.}
    \begin{tabular}{c|ccccccc}
    \hline
         Environment & $f$ & $h$ & $lr$ & $\gamma$ & $N$ & $\alpha$ & $e$    \\
    \hline
         CartpoleBalance & 1 & 4 &3e-5 &0.95 &256 &0.5 &2 \\
         AcrobotSwingup & 1&8 &1e-4 & 0.95 &256 &0.5 &2 \\
         FishSwim &1 &4 &3e-5 &0.95 &256 &0.5 &5 \\
         PointMass &1 &4 &3e-5 &0.95 &256 &0.5 &5 \\
         Go1Getup &5 &32 &3e-5 &0.95 &256 &0.5 &5 \\
        BarkourJoystick &1 &4 &3e-5 &0.95 &256 &0.5 &5 \\
         Go1FootStand &1 &8 &3e-5 &0.95 &256 &0.5 &5 \\
         T1JoystickFlatTerrain &1 &4 &3e-5 &0.95 &256 &0.5 &5 \\
         TRONLocomotion & 1 & 4 & 3e-5 & 0.95 & 256 & 0.5 & 5 \\
    \hline 
    \end{tabular}
    \label{tab:benchmark_params}
    \vspace{-10pt}
\end{table}

\figref{fig:bcmk_result} illustrates the training curves across eight benchmark tasks, showing the progression of episodic average return against the number of environment steps. 
Our benchmark results exhibit consistently improved sample efficiency and better training stability against the PPO and PPO with the RND method across most of the tests, though the degree of improvement varies by task. A noteworthy exception is test \texttt{Go1Getup}, where a quadruped robot needs to learn the fall recovery from an
arbitrary initial state, RND performs better than our method. We attribute this to the highly
discontinuous contact dynamics. Under such task and dynamics, the initial analytical gradients
can be noisy or misleading in our method and the undirected novelty-seeking of RND is proven to
be more effective.
This acceleration against the PPO with plain entropy maximization directly demonstrates the effectiveness of the proposed exploration guidance mechanism. And the comparison with PPO integrated with the novelty-based RND exploration technique shows that ours offers a slight superiority. Beyond the performance comparison, we think that our method is conceptually distinct from the novelty-based RND method, and the value of ours lies in this novel exploration mechanism, which utilizes the underlying dynamics priors of the system to augment the exploration.

\figref{fig:advantage_diff} shows that the exploratory data yields consistently higher advantage, empirically validating \ref{thm:advantage_superiority} and confirming that our mechanism effectively guides the agent toward high-return regions.
Furthermore, our method enhances training stability, evidenced by reduced variance across most tasks; notably, in \texttt{PointMass}, it maintains stable convergence while PPO suffers performance collapse. 
 Finally, sensitivity analysis in \mbox{\figref{fig:sensitivity}} confirms high robustness to update frequency $f$, showing consistent performance across different $f$. Although the horizon length $h$ affects gain magnitude, our method consistently outperforms the PPO baseline across a broad range.
\begin{figure}[htb!]
    \centering
    \includegraphics[width=0.5\linewidth]{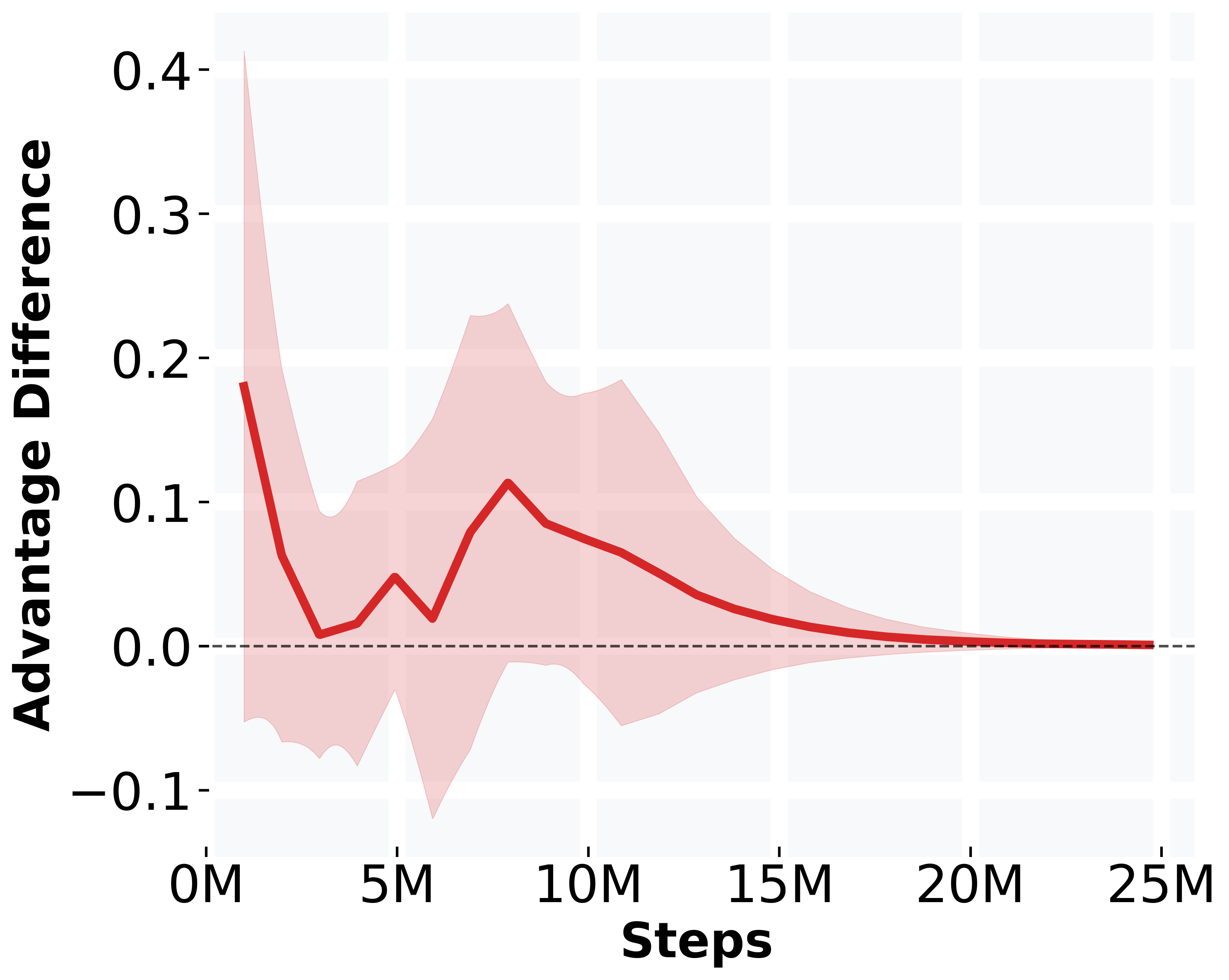}
    \caption{Advantage difference between the data collected by the exploratory policy and that of the primary policy in \texttt{CartpoleBalance} task. Data collected with 5 trials and smoothed with a factor of 0.7.}
    \label{fig:advantage_diff}
    \vspace{-15pt}
\end{figure}

\begin{figure}[htb!]
    \centering
    \begin{minipage}{\linewidth}
    \subfloat{\includegraphics[width=0.5\linewidth]{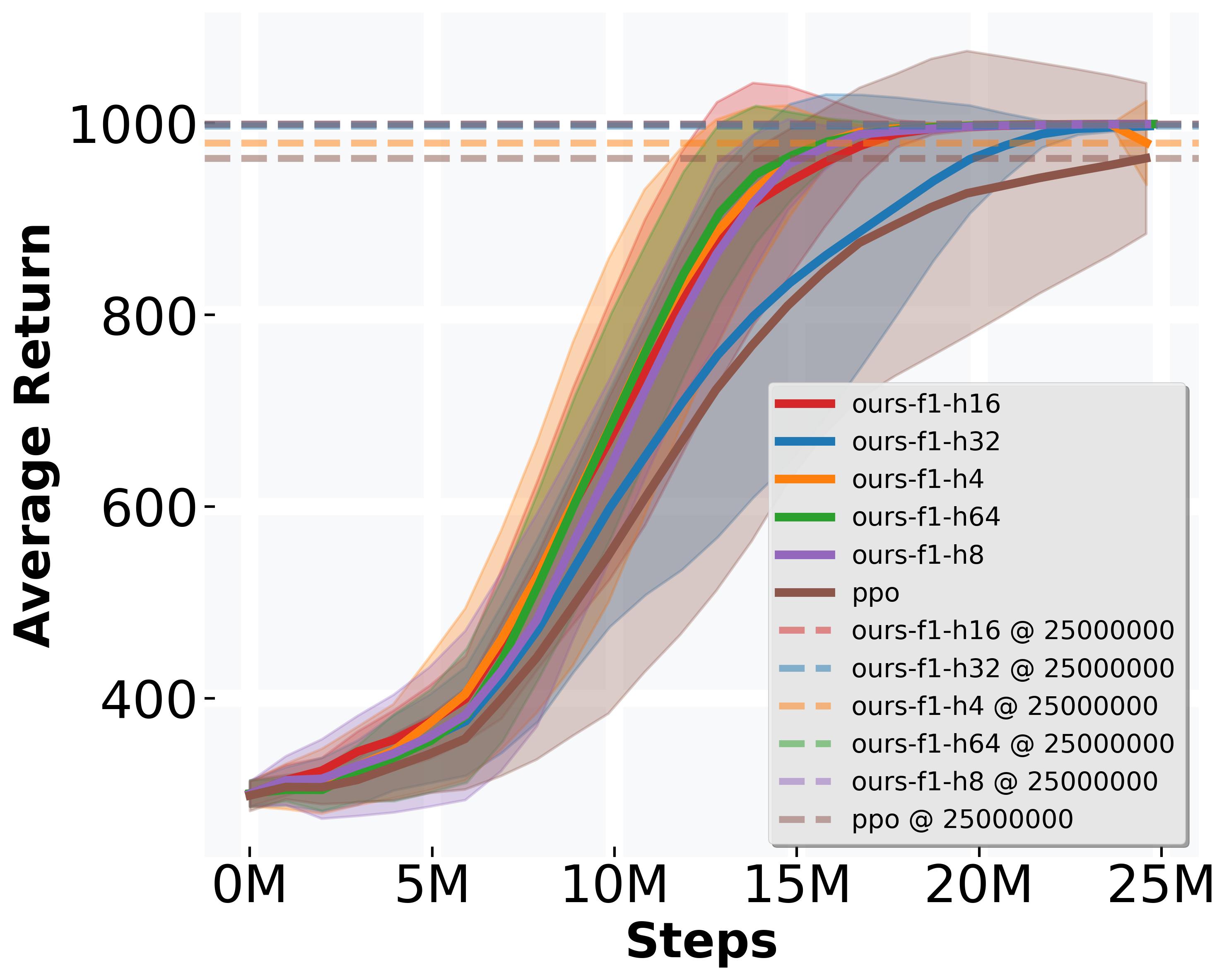}}
    \subfloat{\includegraphics[width=0.5\linewidth]{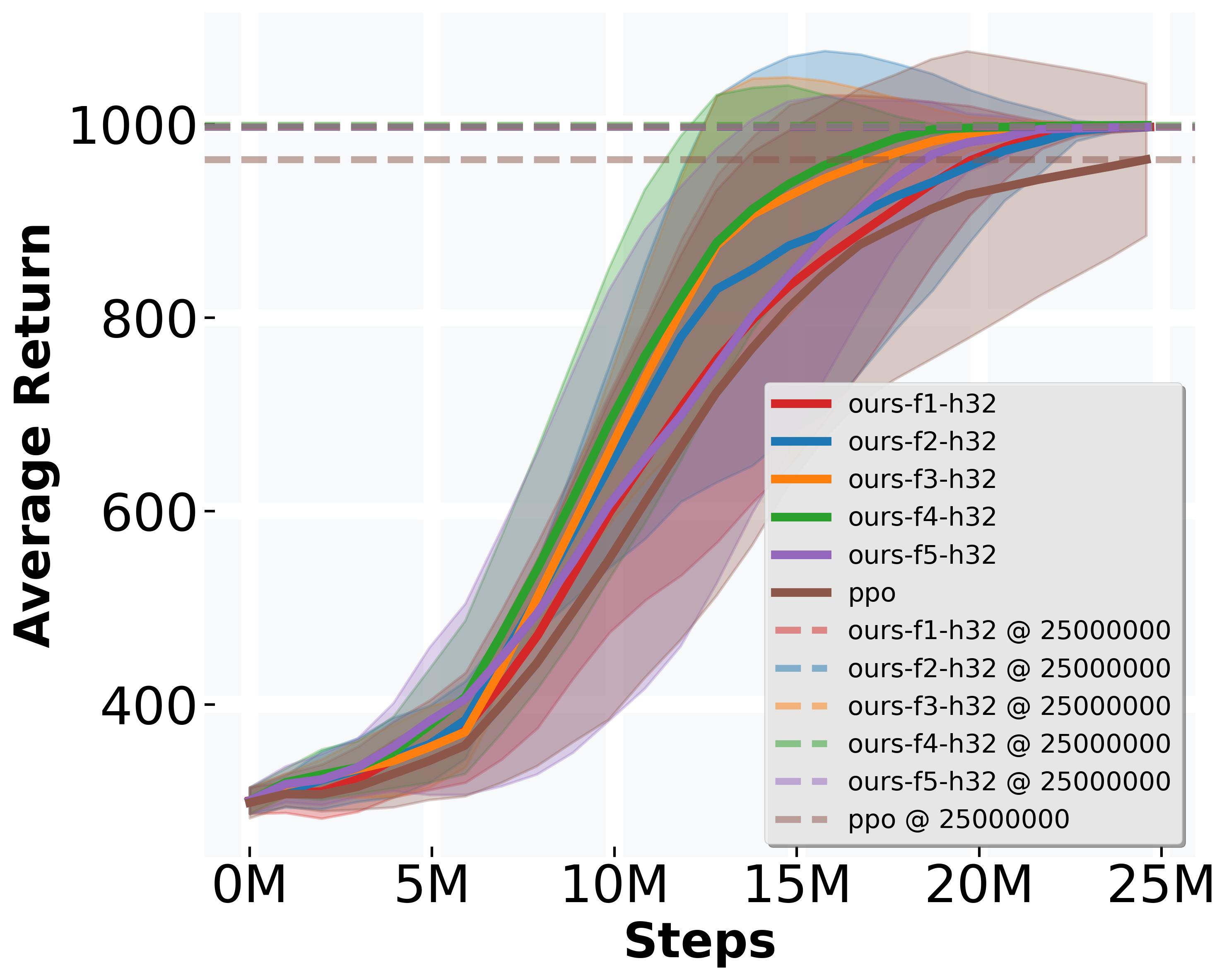}}
    \end{minipage}
    \caption{Sensitivity analysis. Left: APG horizon length $h$; right: for APG update frequency $f$. The other hyperparameter follows the benchmark test except for $N=64$.}
    \label{fig:sensitivity}
    \vspace{-10pt}
\end{figure}

\subsection{Simulation Test and Sim-to-real Physical Validation}
To validate practical viability, we trained a flat-terrain locomotion policy for the LimX TRON 6-DOF biped robot and conducted simulation and sim-to-real experiments.
Specifically, we train velocity-tracking locomotion policies for LimX TRON using both our method and the baseline with reward design shown in \tabref{tab:reward_tron}. 
And we conduct the simulation test to evaluate and compare our performance over the baseline. 
Moreover, we deploy our policy to a real bipedal robot.
\figref{fig:tron_loco_reward} shows the training curves: overall episodic return and specific task rewards defined in the first two rows of \tabref{tab:reward_tron}.
\begin{figure}[t]
    \centering
    \subfloat{\includegraphics[width=0.8\linewidth]{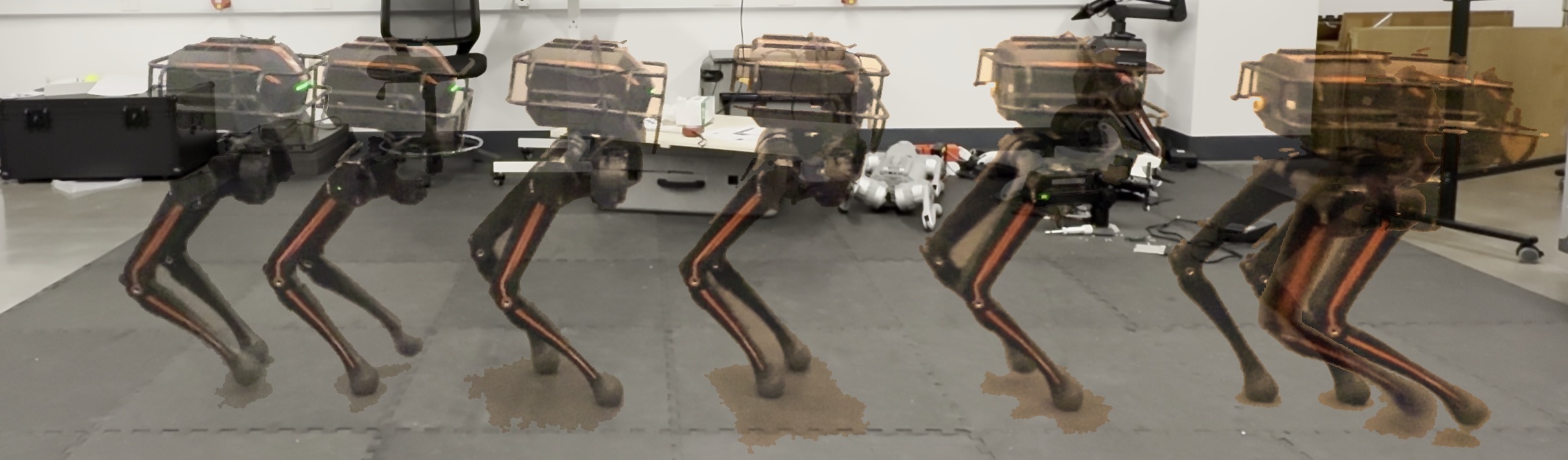}} \\
    \subfloat{\includegraphics[width=0.8\linewidth] {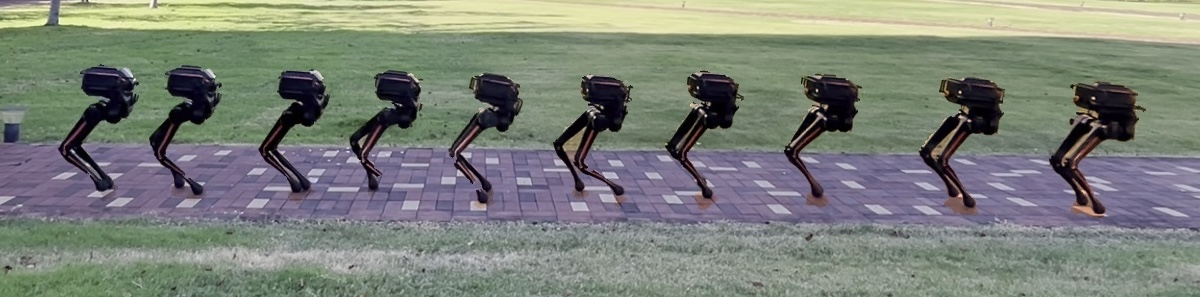}}
    \caption{Biped locomotion sim-to-real transfer experiment with TRON robot. Top: indoor experiment. Bottom: outdoor experiment.}
    \label{fig:sim_to_real}
    \vspace{-10pt}
\end{figure}
The episodic return demonstrates significantly improved sample efficiency, converging almost twice as fast as the baseline to reach similar asymptotic performance.
The velocity tracking reward terms curves on the left show the same trend. This further verifies the effectiveness of our exploration augmentation mechanism for guiding the agent to explore high-reward regions.
\begin{figure}[t]
    \centering
    \begin{minipage}{\linewidth}
    \subfloat{\includegraphics[width=0.5\linewidth]{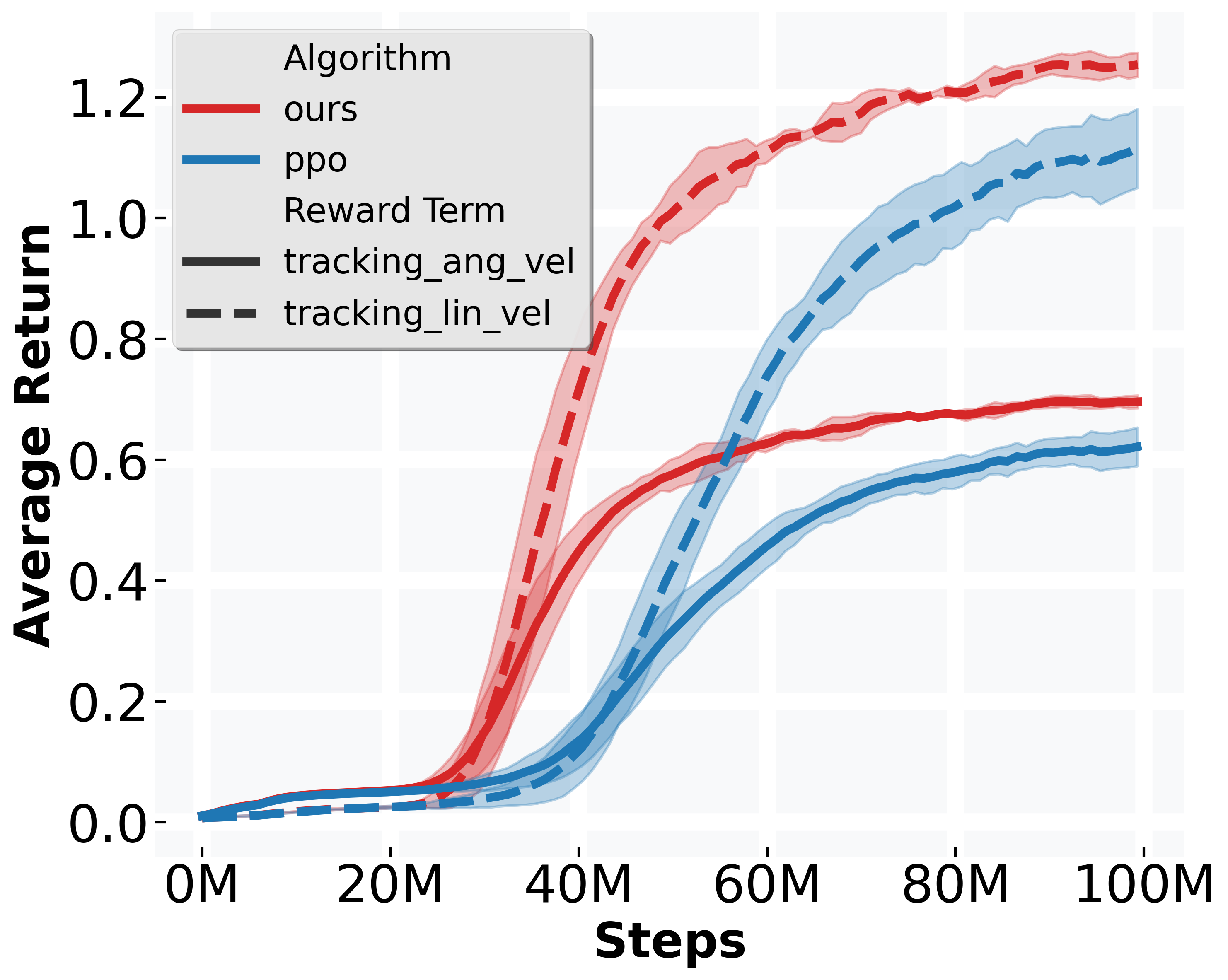}}
    \subfloat{\includegraphics[width=0.5\linewidth]{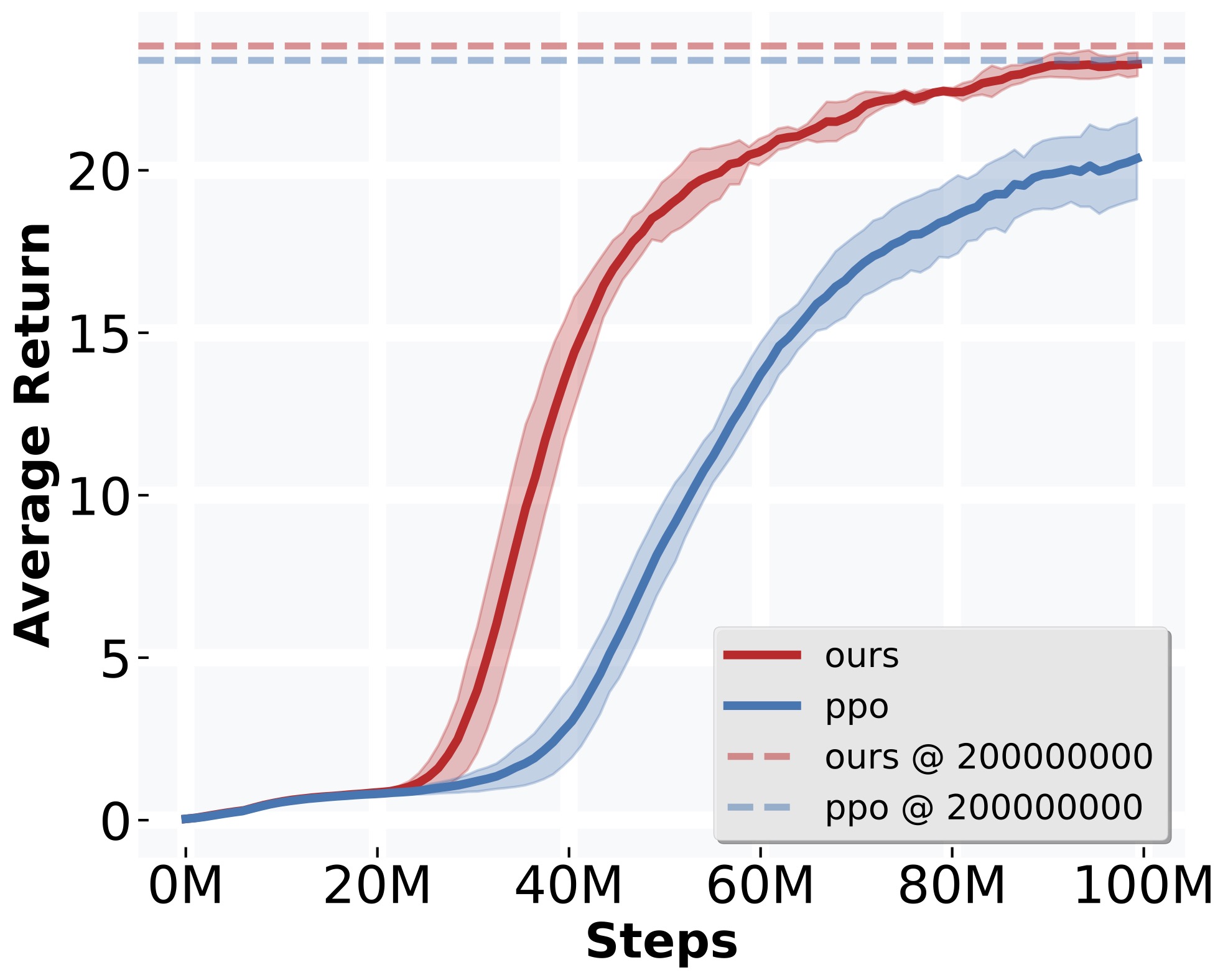}}
    \end{minipage}
    \caption{TRON locomotion policy training curve. Left: angular and linear velocity tracking rewards. Right: overall episodic return.}
    \label{fig:tron_loco_reward}
    \vspace{-10pt}
\end{figure}
\figref{fig:sim_scene} shows the simulation test scene, and \figref{fig:tron_vel_track} compares linear and angular velocity tracking performance under varying velocity commands of our method against PPO, where training reward follows \tabref{tab:reward_tron}.
Our method achieves comparable performance overall, while exhibiting lower tracking error and reduced variance in terms of the velocity tracking task, as displayed in \tabref{tab:vel_track_perform}. 
We attribute improved tracking to APG-directed exploration, which biases the agent toward motion patterns that effectively regulate velocity.
Furthermore, the differentiable task rewards, such as the velocity tracking reward in this case, dominate the analytical gradient and steer trajectories toward task-oriented behaviors that are aligned with the control objective, as indicated in \tabref{tab:vel_track_perform} and \figref{fig:tron_loco_reward}. 
\begin{table}[h]
    \scriptsize
    \centering
    \caption{Velocity tracking performance in the simulation}
    \begin{tabular}{lcc}
    \hline
                                               & Baseline (PPO) & Ours \\ 
    \hline
         MSE of Lin. Vel. Tracking $\text{(m/s)}^2$ & 0.035707  & 0.028816 ($\uparrow$ 19.30\%) \\ \hline
         Error Variances of Lin. Vel. $\text{(m/s)}^2$ & 0.034049 &0.026701 ($\uparrow$ 21.58\%)\\  \hline
         MSE of Ang. Vel. $\text{(rad/s)}^2$ & 0.023856  & 0.009260 ($\uparrow$ 61.18\%)  \\ \hline
         Error Variances of Ang. Vel. $\text{(rad/s)}^2$ &0.017882 &0.008821 ($\uparrow$ 50.67\%) \\ 
    \hline
    \end{tabular}
    \label{tab:vel_track_perform}
    \vspace{-10pt}
\end{table}
\figref{fig:sim_to_real} shows the successful sim-to-real transfer of our biped locomotion policy, demonstrating effective locomotion performance in both indoor and outdoor environments. This deployment serves as the ultimate validation of the proposed algorithm through the lens of practical viability and provides a sanity check for real-world use.
\begin{table}[t]
    \scriptsize
    \centering
    \caption{Reward terms for TRON velocity tracking training.}
    \begin{tabular}{lll}
    \hline
    \textbf{Reward Term} & \textbf{Equation} & \textbf{Weight} \\
    \hline
    Lin. vel. track. & $\exp(-4\|v_{xy}^{\text{cmd}} - v_{xy}\|_2^2)$ & 1.5 \\
    Ang. vel. track. & $\exp(-4(\omega_z^{\text{cmd}} - \omega_z)^2)$ & 0.7 \\
    \hline
    Lin. vel. (z) & $v_z^2$ & $-0.5$ \\
    Ang. vel. (xy) & $\|\omega_{xy}\|_2^2$ & $-0.05$ \\
    Orientation & $\|g_{xy}\|_2^2$ & $-10.0$ \\
    Base height & $(h - h_{\text{target}})^2$ & $-2.0$ \\
    \hline
    Joint acceleration & $\|\ddot{q}\|_2^2$ & $-2.5 \times 10^{-7}$ \\
    Torques & $\|\tau\|_1$ & $-8.0 \times 10^{-5}$ \\
    Action rate & $\|a_t - a_{t-1}\|_2^2$ & $-0.01$ \\
    \hline
    Feet distance & $\text{clip}(d_{\text{min}} - d_{\text{feet}}, 0, 1)$ & $-50.0$ \\
    Feet landing vel. & $\sum \text{landing} \cdot v_{z,i}^2$ & $-0.15$ \\
    Feet air time & $\sum \text{clip}(t_{\text{air}} - t_{\text{min}}, 0, t_{\text{max}})$ & $2.0$ \\
    Feet slip & $\sum \|v_{\text{body}}\| \cdot \text{contact}_i$ & $-0.25$ \\
    Feet phase & $\displaystyle \exp(-\frac{\|h_{\text{feet}} - r_z(\phi)\|_2^2}{0.01})$ & $1.0$ \\
    \hline
    Joint dev. knee & $\sum |q_{\text{knee}} - q_{\text{knee}}^{\text{default}}|$ & $-0.05$ \\
    Joint dev. hip & $\sum |q_{\text{hip}} - q_{\text{hip}}^{\text{default}}| \cdot |\text{cmd}_y|$ & $-0.15$ \\
    DOF pos limits & $\sum \text{violations}(q, q_{\text{lim}})$ & $-2.0$ \\
    Pose & $\sum (q - q_{\text{default}})^2 \cdot w$ & $-1.0$ \\
    \hline
    Termination & $\text{done}$ & $-1.0$ \\
    \hline
    \end{tabular}
    \label{tab:reward_tron}
\end{table} 
\begin{figure}[htb!]
    \centering
    \subfloat{\includegraphics[height=0.25\linewidth]{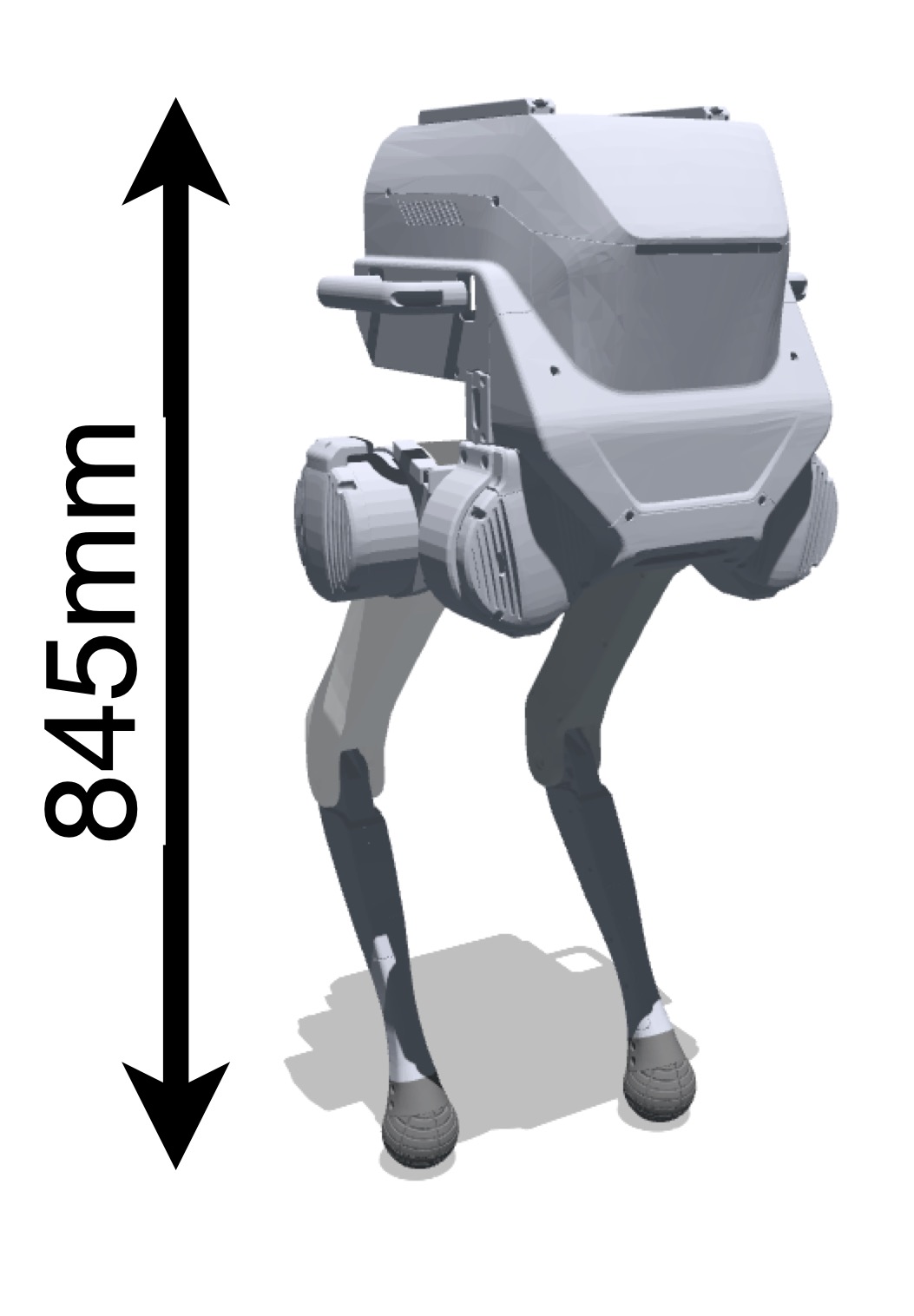}} 
    \subfloat{\includegraphics[height=0.26\linewidth]{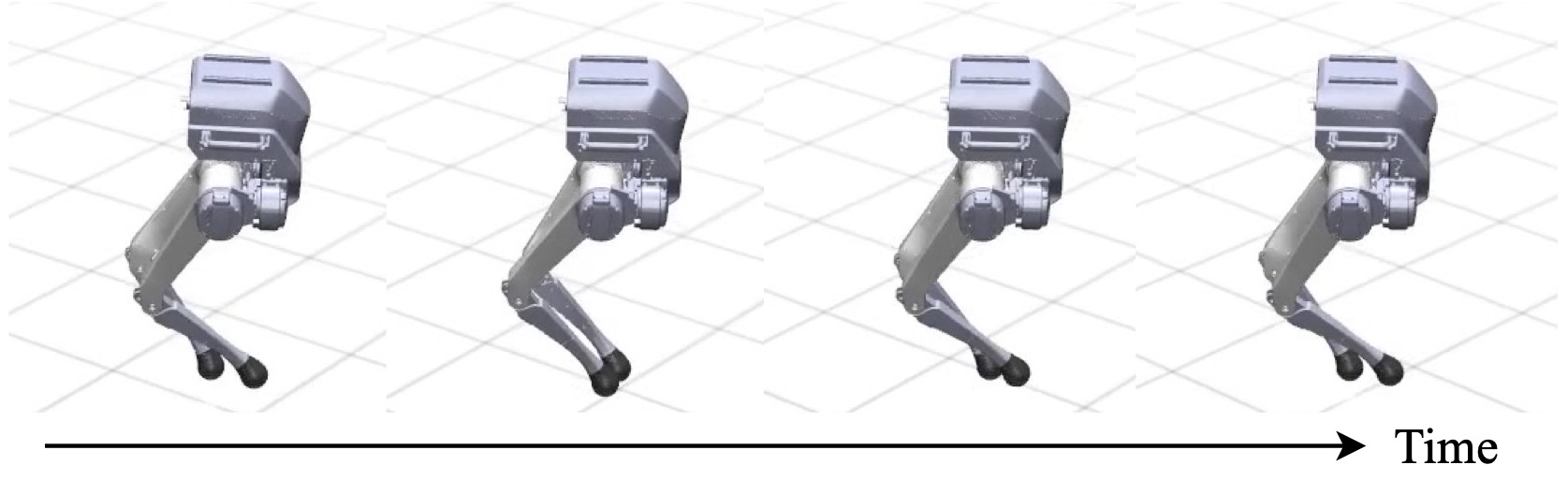}} 
    \caption{Left: 6-DOFs biped robot, LimX TRON. Right: TRON biped locomotion policy test scene with the proposed method.}
    \label{fig:sim_scene}
    \vspace{-10pt}
\end{figure} 
\begin{figure}[t]
    \centering
    \includegraphics[width=\linewidth]{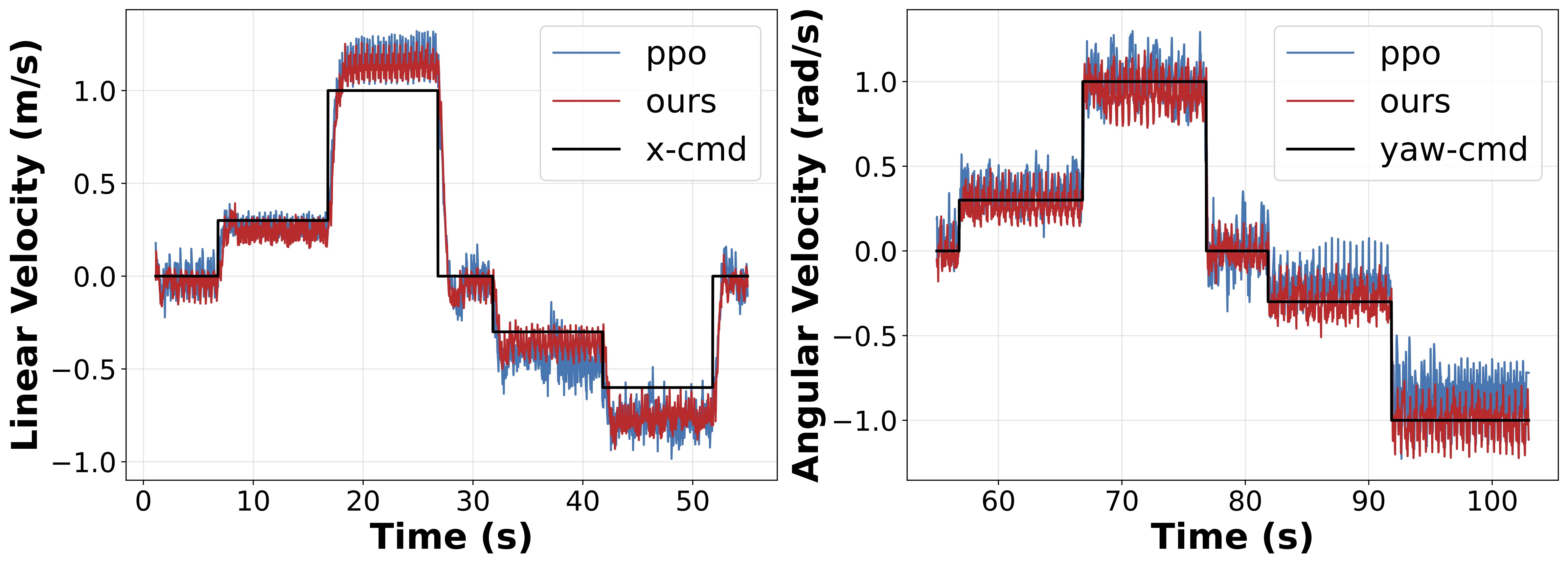}
    \caption{Biped velocity tracking performance in simulation test.}
    \label{fig:tron_vel_track}
    \vspace{-10pt}
\end{figure}

\section{Discussion}
As mentioned in \mbox{\secref{sec:introduction}}, model-free RL accommodates arbitrary rewards but is data-expensive, while FoG methods are sample-efficient but limited to differentiable tasks. Our approach bridges this gap by utilizing analytical gradients solely for exploration guidance. This decoupling potentially enables a hybrid reward structure, where the exploratory policy optimizes dense, differentiable task objectives, such as velocity tracking, while the primary policy optimizes the full reward function, including non-differentiable constraints like gait regularization. Crucially, the model-free update acts as a filter, rejecting exploratory trajectories that violate these constraints via advantage estimation. This design allows our method to exploit dynamics priors for efficiency while preserving compatibility with arbitrary, non-differentiable reward formulations.
Furthermore, while demonstrated here with PPO, a widely adopted RL algorithm for robotic control, our framework is fundamentally a modular data augmentation engine. Since the generated exploratory trajectories are decoupled from the primary policy update logic, this paradigm has the potential to be extended to any actor-critic host RL algorithm, such as A2C or TRPO. 
However, in sparse reward settings where short-horizon gradients vanish, our method leverages the terminal value signal from a well-learned critic. 
Admittedly, during early training with extremely sparse rewards and an uninformative critic, our method may degenerate to PPO. 
Reported in \mbox{\cite{suh2022differentiable, schwarke2024learning}}, analytical gradients in contact-rich tasks often exhibit empirical bias. A dynamics-based curriculum can mitigate this by starting with smoothed dynamics to provide reliable early guidance, before progressively increasing simulation realism.
Finally, while offering better wall-clock efficiency in simpler dynamics, our method may currently underperform PPO in complex contact-rich tasks due to the differentiable simulator's implementation overhead.

\section{Conclusion and Future Work} \label{sec:conclusion}
We proposed an exploration augmentation framework utilizing analytical policy gradients to guide on-policy RL. Theoretical analysis confirms that this strategy yields a positive-biased advantage, driving significantly improved sample efficiency.
Extensive benchmarks and sim-to-real deployment on a bipedal robot demonstrate consistent gains in training stability and physical viability, validating the method's effectiveness across diverse tasks.
Our approach uniquely combines gradient-guided exploration with flexible, non-differentiable rewards, though it remains sensitive to sparse reward settings and simulator fidelity.
Future directions include adaptively tuning the data merging ratio based on gradient reliability and integrating learned dynamics models to extend applicability to complex, black-box environments.

\bibliographystyle{IEEEtran}
\bibliography{IEEEabrv, ref}

\end{document}